\documentclass[leqno,12pt,final]{article}

\usepackage{ifxetex}

\newcommand{\s}[1][1]{\hspace{#1pt}}

\newcommand{\ones}{\rme}

\usepackage{pgffor}

\foreach \x in {A,B,...,Z,a,b,...,z} 
{
  \expandafter\xdef\csname cal\x\endcsname{\noexpand 
	\ensuremath{\noexpand\mathcal{\x}}}
  \expandafter\xdef\csname scr\x\endcsname{\noexpand 
	\ensuremath{\noexpand\mathscr{\x}}}
  \expandafter\xdef\csname bb\x\endcsname{\noexpand 
	\ensuremath{\noexpand\mathbb{\x}}}
  \expandafter\xdef\csname rm\x\endcsname{\noexpand 
	\ensuremath{\noexpand\mathrm{\x}}}
  \expandafter\xdef\csname bf\x\endcsname{\noexpand 
	\ensuremath{\noexpand\mbf{\x}}}
}

\ifxetex
\let\tau\uptau
\let\alpha\upalpha
\let\zeta\upzeta
\let\kappa\upkappa
\let\omega\upomega
\else
\fi

\usepackage{color}

\definecolor{red}{rgb}{1,0,0}
\definecolor{gray}{rgb}{0.5,0.5,0.5}
\definecolor{darkgray}{rgb}{0.4,0.4,0.4}
\definecolor{blue}{rgb}{0,0,1}
\definecolor{green}{rgb}{0,1,0}

\definecolor{deluge}{RGB}{124, 113, 173}
\definecolor{bamboo}{RGB}{220, 92, 5}
\definecolor{yellow}{RGB}{255, 172, 0}
\definecolor{orange}{RGB}{255, 144, 0}
\definecolor{oyster}{RGB}{151, 139, 125}
\definecolor{coral}{RGB}{199, 186, 167}
\definecolor{downy}{RGB}{110, 197, 184}

\title{{\Large \bf On Minimum Trace Factor Analysis} \\
\large \textsc{An Old Song Sung to a New Tune}}

\author{
Chang Yuan Li\footnote{Department of Statistics 
and Applied Probability, University of California, 
Santa Barbara. Email:~{\tt cli@pstat.ucsb.edu}.}
\hspace{0.32in} and \hspace{0.32in} 
Alex Shkolnik\footnote{Department of Statistics 
and Applied Probability, University of California, 
Santa Barbara, CA and Consortium for Data Analytics in Risk, 
University of California, Berkeley, CA.
Email:~{\tt shkolnik@ucsb.edu}.}
}

\date{\today}

\usepackage[lofdepth,lotdepth,caption=false]{subfig}
\usepackage{fancyhdr}
\usepackage[hyperfootnotes=false]{hyperref}
\usepackage{amsmath,amssymb,amsthm,mathrsfs,mathtools,bm, bbm,graphicx,verbatim, tikz, tablefootnote, kpfonts}
\usepackage{xspace}
\usepackage{braket}
\usepackage{color}
\usepackage{setspace}
\usepackage{algorithm}
\usepackage{algpseudocode}
\usepackage{breakcites}

\newcommand{\norm}[1]{\lVert#1\rVert}
\newcommand{\vertiii}[1]{{\left\vert\kern-0.25ex\left\vert\kern-0.25ex\left\vert #1 
    \right\vert\kern-0.25ex\right\vert\kern-0.25ex\right\vert}}

\newcommand{\proj}[1]{\text{P}_{#1}}
\newcommand{\defeq}{\vcentcolon=}
\DeclareMathOperator*{\minimize}{minimize\;}
\DeclareMathOperator*{\maximize}{maximize\;}
\DeclareMathOperator*{\argmin}{argmin}

\DeclareMathOperator*{\diag}{diag}
\DeclareMathOperator*{\rank}{rank}
\DeclareMathOperator*{\prox}{prox}
\DeclareMathOperator*{\vectorize}{vec}
\DeclareMathOperator*{\var}{Var}
\DeclareMathOperator*{\e}{E}

\DeclareMathOperator*{\tr}{tr}
\DeclareMathOperator*{\nul}{null}
\DeclareMathOperator*{\col}{col}
\DeclareMathOperator*{\row}{row}

\DeclareMathOperator*{\pdiag}{\proj{diag}}
\DeclareMathOperator*{\poffdiag}{\proj{diag}^\perp}

\usepackage{tcolorbox}
\theoremstyle{plain}
\newtheorem{thm}{Theorem}
\theoremstyle{definition}
\newtheorem{defn}[thm]{Definition} 

\newtheorem{cor}[thm]{Corollary}
\newtheorem{lmm}[thm]{Lemma}

\newtheorem{prop}[thm]{Proposition}
\newtheorem{remark}[thm]{Remark}
 \newtheorem{assmpt}[thm]{Assumption}
\hypersetup{
    colorlinks = true,
    urlcolor = cyan,
    linktoc = none
}

\begin{document}

\maketitle

\begin{abstract}
 Dimensionality reduction methods, such as principal component
analysis (PCA) and factor analysis, are central to many problems
in data science. There are, however, serious and well-understood
challenges to finding robust low dimensional approximations for
data with significant heteroskedastic noise. This paper
introduces a relaxed version of Minimum Trace Factor Analysis
(MTFA), a convex optimization method with roots dating back to
the work of Ledermann in 1940.  This relaxation is particularly
effective at not overfitting to heteroskedastic perturbations
and addresses the commonly cited Heywood cases in factor
analysis and the recently identified ``curse of
ill-conditioning" for existing spectral methods.  We provide theoretical
guarantees on the accuracy of the resulting low rank subspace and the convergence rate of the proposed algorithm
to compute that matrix. We develop a number of interesting
connections to existing methods, including Hetero PCA,
Lasso, and Soft-Impute, to fill an important gap in the already
large literature on low rank matrix estimation. Numerical
experiments benchmark our results against several
recent proposals for dealing with heteroskedastic noise.
\end{abstract}

\section{Introduction}

Minimum trace factor analysis (MTFA) is a long-existing
statistical method that is related to factor analysis and dates
back to the work of \cite{ledermann_iproblem_1940}. Its
fundamental objective is to extract the ``largest" diagonal
matrix $\scrD$ from a covariance matrix $\Sigma$ while ensuring that
the residual matrix $\Sigma -  \scrD$ remains 
symmetric and positive semi-definite.  The subtraction of this diagonal
matrix $\scrD$ may be viewed as a factor analysis that
minimizes the number of common factors (i.e., the rank of 
$\Sigma - \scrD$) that specify a
covariance matrix, resulting in a more parsimonious model.

Let $\bbS^p$ space of all real symmetric $p \times p$ matrices.
Given an input $\Sigma \in \bbS^p$, MTFA solves the 
following optimization problem.
\begin{align}
\begin{aligned}
    \minimize_{\mathscr L, \mathscr D} \quad&\tr(\mathscr L) \\
\text{s.t. }\quad &\scrL = \Sigma - \scrD  \in\mathbb{S}^p_+,\\
&\mathscr D= \pdiag(\mathscr D).
\end{aligned}
\label{eqn:mtfa}
\end{align} 
The feasible set $\bbS^p_+ \subset \bbS^P$ denotes
the space of positive semi-definite $p \times p$ matrices,
while $\pdiag(M)$ is the matrix $M$ with all but the
diagonal cells set to zero. Thus, $M =\pdiag(M)$ implies 
$M$ is a diagonal matrix. We further set
$\poffdiag(M) \defeq M - \pdiag(M)$.

The pioneering work in \cite{ledermann_iproblem_1940}
revealed that the minimum trace optimization problem
\eqref{eqn:mtfa} does not generally lead to a minimum rank
decomposition, i.e., for $D$ diagonal, 
$\Sigma = L + D$ may have $L \in \bbS^p_+$ be of lower
rank than the first variable of the minimizer. This might
have curtailed the application of MTFA to factor analysis,
where the rank of $L$ equates to the number of factors
in a factor analysis of a covariance. Decades later,
MTFA was rediscovered in
\cite{bentler_lower-bound_1972}, particularly in the context of
determining the greatest lower bound of reliability in factor
analysis. For example, it is common practice in psychology to report
reliability coefficients, i.e., 
\[ \frac{\ones^\top L\s \ones}{\ones^\top \Sigma \ones}= 
1 - \frac{\tr(D)}{\ones^\top \Sigma \ones}\]
is that coefficient for the model $\Sigma = L + D$ above,
and $\ones$ is a $p$-dimensional vector of all ones.
Therefore, MTFA offers a straightforward approach to producing
a lower bound on the reliability coefficient. This method was
further investigated 
in \cite{shapiro_rank-reducibility_1982} with respect to its rank
reducibility, addressing the question of how much the rank can
be reduced by modifying the diagonal elements (see Proposition \ref{prop:rank_reducibility} below). Other studies have also delved into various aspects, including \cite{fletcher_nonlinear_1981,della_riccia_minimum_1982,ten_berge_computational_1981,shapiro_weighted_1982,cronbach_internal_1988,bekker_generic_1997,jamshidian_quasi-newton_1998,shapiro_asymptotic_2000,shapiro_statistical_2002, saunderson_diagonal_2012,ning_linear_2015, bertsimas_certifiably_2017,shapiro_statistical_2019}. 

Much more recently, nuclear norm optimization has gained
significant attention in computer science and modern statistics
literature, for this convex surrogate provides a computationally
feasible heuristic for rank minimization
\cite{fazel_matrix_2002}. It offers statistical and
computational properties in area such as robust PCA
\cite{candes_robust_2011} and matrix completion
\cite{candes_exact_2009}, which makes it popular in image/video
processing \cite{bouwmans_handbook_2016} and recommendation
systems \cite{mazumder_spectral_2010}. MTFA can also be
equivalently expressed as a nuclear norm minimization problem,
replacing the trace in $\eqref{eqn:mtfa}$ by a 
nuclear norm. Because nuclear norm is the convex surrogate of the rank, MTFA tends to have
a lower rank among the elements of the feasible set $\bbS^p_+ \cap \{M: \poffdiag(M) = \poffdiag(\Sigma)\}$.


\begin{lmm}
For $\scrL \in \mathbb{S}^p_+$ and
$\norm{\s[2] \cdot \s[2] }_*$, the nuclear norm (ie., the 
sum of the singular values of the argument),
\[\norm{\mathscr L}_* =  \tr(\mathscr L)\]
\end{lmm}

In the past decade, a fresh analysis of MTFA was
presented in \cite{saunderson_diagonal_2012}, revealing
significant promise for this method through the exploration of
subspace coherence -- a recurring concept in the realms of
compressed sensing, robust PCA and other actively researched
areas. 

\begin{prop}[\cite{saunderson_diagonal_2012}]\label{prop:mtfa_exact} 
Let $\Sigma = L+D$ for $L \in \mathbb{S}^p_+$ of rank $r$ and $D$ 
a diagonal matrix. Denote the column space of $L$ by $\scrU$. 
The following are equivalent.
	      	\begin{enumerate}
	      		\item $\mathscr{U}$ is recoverable by MTFA,
i.e., $(L, D)$ is the unique minimizer of $\eqref{eqn:mtfa}$. 
	      		\item $\mathscr U$ is realizable, i.e.~there is a correlation matrix $Q\in \mathbb{S}^p_+$ such that $\nul (Q) \supseteq \mathscr U$. 
	      		\item $\mathscr U^\perp$ has the ellipsoid fitting property, i.e.~there is a $(p-r)\times p$ matrix $M$ with row space $\mathscr U^\perp$ and a centered ellipsoid in $\mathbb{R}^{p-r}$ passes through each column of $M$.
	      	\end{enumerate}
	      \end{prop}

The connection between realizability and subspace coherence is given below. 

        \begin{defn}[Coherence of a subspace]\label{defn:coherence}
        Let $\mathscr U$ be a $r$-dimensional subspace of $\mathbb{R}^p$, then
             \begin{align}
             \hbar(\mathscr U)\defeq \max_{1\leq i\leq p} \norm{\proj{\mathscr{U}}e_i}^2 = \max_{u\in \mathscr U\backslash \{0\}} \frac{\norm{u}_\infty^2}{\norm{u}^2}= \norm{U}_{2, \infty} ^2 \in [r/p, 1]
             \end{align}
             where $U\in\mathbb{O}^{p, r} = \{\mathscr U\in \mathbb{R}^{p\times r}: \mathscr U ^\top \mathscr U = I_r \}$ has column space $\mathscr U$ and $\norm{\cdot}_{a, b}$ is the $(a, b)$-norm that is defined as follows: it first computes the $l_a$-norm for each row, and then computes the $l_b$-norm of the resulting vector.
        \end{defn}
	\begin{prop}[\cite{saunderson_diagonal_2012}]\label{prop:balanced} If $\mathscr U$ is a subspace of $\mathbb{R}^p$ and  $\hbar(\mathscr{U})  < 1/2$, then $\mathscr U$ is realizable. On the other hand, given any $\alpha > 1/2$, there is a subspace $\mathscr U$ with $\hbar(\mathscr{U}) = \alpha$ that is not realizable.\end{prop}

       \begin{cor}[\cite{ledermann_iproblem_1940,saunderson_diagonal_2012}]\label{cor:balanced}
        
         When $L = \sigma^2 \beta\beta^\top$ for $\sigma \in
\bbR$
and $\beta \in \bbR^p$, the row space of $L$ is recoverable by
MTFA if and only if the eigenvector $\beta$ of $L$ is balanced, i.e.
	      \begin{align*}
	    \s[32]  	|\beta_i|\leq\sum_{j\neq i}|\beta_j| 
 \s[32] \forall \s i=1, \cdots, p \s . 
	      \end{align*}
       \end{cor}

The preceding results effectively elucidate the advantages
of MTFA over PCA as the former provides exact recovery guarantees. To
gain a more comprehensive understanding of MTFA and PCA,
particularly regarding PCA's susceptibility to outliers, let's
examine a straightforward example. Although the literature
frequently asserts the ``sensitivity of PCA to outliers", it
often lacks a rigorous mathematical depiction of the underlying
mechanisms, relying instead on empirical observations.

It is customary to take the sample covariance $\hat\Sigma$ as an input of PCA and MTFA, as the true covariance matrix  $\Sigma= L+D$ is unobservable.  To facilitate notation, let us define the matrix $W \in \mathbb{R}^{p \times p}$ as the ``discrepancy,'' expressed as:  $W\defeq \hat \Sigma - \Sigma\in\mathbb{R}^{p\times p}$. This allows us to represent $\hat{\Sigma}$ as the sum of the true covariance matrix $L + D$ and the discrepancy $W$:
\begin{align}
    \hat\Sigma = \Sigma + W = L+ D + W 
\end{align}

\subsection{Motivating example}
\label{sec:example}

Let's consider the 1-factor model with noiseless ($W = 0$) input covariance matrix

\begin{align*}
     \Sigma &=\hat\Sigma =\sigma^2_\text{signal} \beta\beta^\top + D\\ D& \defeq \delta^2_{\text{noise}}I+\sigma^2_{\text{noise}}\eta\eta^\top 
\end{align*}
where $\norm{\beta} = 1$ and $\eta_i = 1$ for an index $i$ and $\eta_j = 0$ for $j \neq i$.
Given the true covariance matrix $\Sigma$, what would happen if we attempt to estimate $\beta$ using PCA?  Let $q= \beta^\top \eta, s = \sigma^2_{\text{signal}}/\sigma^2_{\text{noise}}$. The first eigenvector is proportional to the following form \[\hat\beta \propto \begin{cases}\beta + \frac{1-q^2 + \sqrt{(1-s)^2+4sq^2}}{1+s-\sqrt{(1-s)^2+4sq^2}}\cdot \frac{\eta}{q}, & q\neq 0;\\
1(s > 1)\beta + 1(s<1)\eta, & q = 0.\end{cases}\]
In terms of $\sin\Theta$ distance,
\begin{align*}
    |\sin\Theta(\beta, \hat\beta)|\defeq \sqrt{1-\left(\tfrac{\beta^\top \hat\beta}{\norm{\beta}\cdot \norm{\hat\beta}}\right)^2} = \begin{cases}\left[ 1+ \tfrac{4q^2(1-q^2)^2}{\left(1-s-2q^2 + \sqrt{(1-s)^2 + 4s q^2
    }\right)^2}\right]^{-1/2}, & q\neq 0;\\
    1(s<1), & q = 0.\end{cases}
\end{align*}
As displayed in Figure \ref{fig:example_sin_theta}, one can observe a phase transition phenomenon when we plot the $\sin\Theta$ distance with respect to the signal heteroskedasticity ratio $\sigma^2_\text{singal}/\sigma^2_\text{noise}$.   On the other hand, under the condition that $\beta$ is balanced,  Corollary \ref{cor:balanced} says that one can perfectly recover the low-rank covariance matrix.

\begin{figure}[H]
    \centering
    \includegraphics[width=0.5\linewidth]{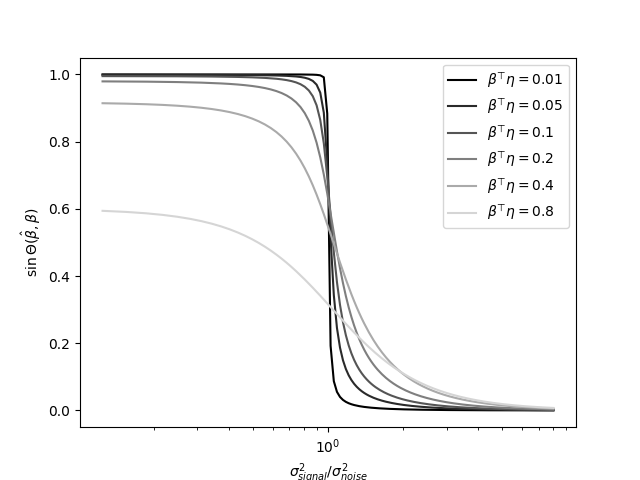}
    \caption{Rank 1 subspace recovery by PCA while heteroskedasticity exists}
    \label{fig:example_sin_theta}
\end{figure}

Ideally, the sample covariance matrix can be close to the true covariance matrix. However, owing to the constraint $\poffdiag(\Sigma) = \poffdiag(\mathscr L)$, when one observes a sample covariance matrix $\hat \Sigma$, which is the true covariance matrix perturbed by some noise, MTFA provides a solution with a linearly increasing rank with respect to its dimension, as demonstrated in Proposition \ref{prop:rank_reducibility}. This could be one of the contributing factors to the method's relative obscurity in the past.

\begin{prop}[\label{prop:rank_reducibility}\cite{shapiro_rank-reducibility_1982}] The Lebesgue measure of all symmetric matrix $\Sigma\in \mathbb{S}^p$ such that the corresponding MTFA solution low-rank component $\tilde L = \Sigma - \tilde{D}$ has rank less than \begin{align*}
	    \phi(p)\defeq \frac{2p+1-(8p+1)^{1/2}}{2}
	\end{align*}
	is zero. 
	\end{prop}

Due to Proposition \ref{prop:rank_reducibility}, the MTFA program's solutions, when provided with sample covariance inputs, fail to yield low-dimensional outcomes, instead, the rank of the solution grows linearly with dimension $p$. 

Additionally, MTFA is susceptible to another issue, often resulting in improper solutions. In such cases, the covariance matrix decomposition $\Sigma = L + D$ includes a diagonal matrix $D$ with 0 or negative elements on its diagonal. This situation poses a significant challenge for interpretation since variances are inherently expected to be positive. This issue is commonly referred to as ``Heywood cases'' within the factor analysis literature and practice.

To address the Heywood cases, the constrained MTFA (CMTFA) \cite{jackson_lower_1977, bentler_inequalities_1980} imposes the constraint that $D$ must be positive semi-definite. However, as pointed out in \cite{bartholomew_latent_2011}, it is important to acknowledge that one of the most common causes of Heywood cases is the attempt to extract more factors than are genuinely present. Therefore, while imposing $D\in\mathbb{S}^p_+$ is a step in the right direction, it may not always be the most effective approach to resolve the problem.

\subsection{Contributions}
\label{sec:contrib}

To overcome the limitations, we propose a relaxed version of MTFA in this paper. Our enhanced MTFA approach introduces several key advancements:

\begin{itemize}
    \item {\bf Iterative fixed point method with convergence rate guarantees}: We have developed an iterative fixed-point method with convergence rate guarantees, a valuable contribution to the factor analysis literature where such guarantees are not commonly found.
    \item {\bf Optimal subspace recovery guarantees}:
   Our approach establishes optimal subspace recovery guarantees, aligning with the minimax rate within the factor model setting.  
   \item {\bf Overcoming the ``Curse of Ill-Condition'' with even milder condition}:
   As highlighted by \cite{zhou_deflated_2023}, existing methods for subspace estimation degrade in the presence of heteroskedastic noise when the condition number $\kappa$ associated with the signal becomes large. For instance, Diagonal-Deleted PCA \cite{cai_subspace_2021} Theorem 1, and HeteroPCA \cite{zhang_heteroskedastic_2022} Theorem 4, \cite{yan_inference_2021} Theorem 10, and \cite{agterberg_entrywise_2022} Assumption 4 require
   $$\kappa \lesssim p^{1/4}$$
   under model \eqref{eqn: matrix_model} and the statistical bound are all involving $\kappa$. Through the exploitation of convex properties, our method does not require a condition number and imposes less stringent conditions compared to \cite{zhou_deflated_2023}.
   \item  {\bf Addressing Heywood cases in Factor Analysis}: Our method provides a guarantee of preventing Heywood cases for suitably regularization.
    \item {\bf Connections to existing literature}:
    
    Our work establishes several meaningful connections to the existing literature: 
    \begin{itemize}
        \item It presents a relaxation of MTFA of the exact constraint, and a convex relaxation of the HeteroPCA algorithm \cite{zhang_heteroskedastic_2022}.
        
        \item By incorporating concepts from LASSO \cite{tibshirani_regression_1996} into the domain of factor analysis, our approach bridges the gap between these two important methodologies.
    
        \item It establishes a connection between MTFA and  missing value imputation techniques, specifically Soft-impute \cite{mazumder_spectral_2010} and covariance estimation with missing values \cite{lounici_high-dimensional_2014}.
    \end{itemize} 
\end{itemize}

\section{Proposed Mathematical Program and its Properties}
We present the following optimization problem involving a given  covariance matrix $\Sigma$ and a tuning parameter $\tau>0$:

\begin{align}
\begin{aligned}
\minimize_{\mathscr L, \mathscr D}\quad & \,F(\mathscr L, \mathscr D)\defeq \tau\norm{\mathscr L}_* + \frac{1}{2}\norm{\Sigma 
 - (\mathscr L + \mathscr D) }_F^2 \\
\text{s.t. } \quad
&\mathscr L \in\mathbb{S}^p_+,\\
&\mathscr D= \pdiag(\mathscr D). \end{aligned}
\label{eqn:rmtfa}
\end{align}
In this context, the objective function $F$ can be regarded as proportional to the Lagrangian function of the following constrained optimization problem:
 \begin{align}
        \begin{aligned}
        \minimize_{\mathscr L, \mathscr D}\quad & \,\norm{\mathscr L}_* \\
        \text{s.t. } \quad
        &\norm{\Sigma 
         - (\mathscr L + \mathscr D) }_F^2 \leq \psi,\\
        &\mathscr L \in\mathbb{S}^p_+,\\
        &\mathscr D= \pdiag(\mathscr D). \end{aligned}\label{eqn: constrained_form1}
        \end{align}
In optimization language, \eqref{eqn:rmtfa} is the dual problem and \eqref{eqn: constrained_form1} is the primal problem. The equivalence between the two is established in
Lemma \ref{lmm:two_optimization}. %

To interpret \eqref{eqn: constrained_form1}, we relax the MTFA problem represented by \eqref{eqn:mtfa}, such that the exact constraint $\Sigma = \mathscr L + \mathscr D$ is not strictly enforced. Rather, $\mathscr L + \mathscr D$ is close to $ \Sigma$ in the least square sense. This relaxation allows for noise presented in the covariance matrix $ \Sigma$. The parameter $\psi$ in \eqref{eqn: constrained_form1} is required to be non-negative, and controls the distance between the decomposition and the input.

It is evident that our proposed program constitutes a convex optimization problem. Nevertheless, within the literature on Factor Analysis, such instances are rather infrequent. For instance, extensions of minimum residual factor analysis \cite{harman_factor_1966}, often involve rank-constrained problems that are challenging to solve, as highlighted by Bertsimas \cite{bertsimas_certifiably_2017} who resorts to discrete optimization to provide a certifiable solution. Furthermore, the likelihood-based approach imposes distributional assumptions on the residuals, which is a more restrictive requirement. In contrast, our proposed method refrains from assuming a specific distributional assumption and guarantees a unique global minimizer.

\begin{lmm}\label{lmm:cvx_uniqueness} The optimization \eqref{eqn:rmtfa} is convex and has a unique global minimizer as $\tau > 0$.
\end{lmm}


We also establish a connection between MTFA and our optimization problem \eqref{eqn:rmtfa}. It can be demonstrated that for $\tau\geq \lambda_1^{(\poffdiag(\Sigma))}$, the solution $\hat L(\tau)$ becomes a zero matrix with a rank of 0. As the tuning parameter $\tau$ approaches 0, we have demonstrated that (Theorem \ref{thm:rmta_mtfa}) $\hat L(\tau)$ converges towards the MTFA solution, indicating that $\hat L(\tau)+\hat D(\tau)$ closely approximates $\Sigma$  and $\rank(\hat L(\tau))$ tends to be high  (Proposition \ref{prop:rank_reducibility}). The monotonic relationship between nuclear norm and $\tau$ is established in Lemma \ref{lmm:monotone}, $\tau$ serves as a user-defined tuning parameter to accommodate the prevailing noise and achieve a desired rank. From now on, we define \[\psi(\tau) \defeq \norm{\Sigma - (\hat L(\tau)+\hat D(\tau))}_F^2. \]

\begin{lmm}
    \label{lmm:monotone}
    
        For $\tau \in (0, \lambda_1^{(\poffdiag(\Sigma))})$, 
        $\psi(\tau)$ is a strictly increasing function  with
        $\psi(\tau) \leq \tau \norm{\Sigma}_*;$ whereas $\norm{\hat L(\tau)}_*$ is a strictly decreasing function. 
    \end{lmm}

\begin{thm}[MTFA solution as a limit point\label{thm:rmta_mtfa}]
    Given the covariance matrix $\Sigma$ as input. Let $\hat L(\tau)$ be the solution to \eqref{eqn:rmtfa} with tuning parameter $\tau>0$ and $\tilde L$ be the solution to MTFA \eqref{eqn:mtfa}. 
    \begin{align*}
       \hat L(\tau) \rightarrow \tilde L  \quad \text{as ${\tau\rightarrow 0^+}$.}  
\end{align*}
\end{thm}

When noise is present, achieving the exact recovery guarantees as MTFA does becomes unattainable. However, we are able to establish a robust $\sin\Theta$ theorem that allows for the recovery of the subspace. The analysis in Theorem \ref{thm: robust_sin_rmtfa} is deterministic, akin to the Theorem 3 in the paper \cite{zhang_heteroskedastic_2022}. Building on the insights from \cite{zhang_heteroskedastic_2022}, our method also attains the minimax rate under the factor model assumption. Furthermore, leveraging the property of convexity, we can effectively address the challenge of ill-conditioning highlighted by Zhou \cite{zhou_deflated_2023}, where a substantial spectral gap can present difficulties for non-convex methods. We can solve the problem with fewer restrictive requirements.

We employ the $\sin\Theta$ distance to measure dissimilarity between two $p\times r$ matrices: $U$ and $\hat U$. Both matrices belong to $\mathbb{O}^{p, r}$, comprising columns that pertain to orthogonal matrices. Let $U_\perp, \hat U _ \perp \in \mathbb{O}^{p, (p-r)}$ be defined such that $[U, U_\perp]$ and $[\hat U, \hat U_\perp]$ form orthonormal matrices. Denote the first $r$ singular values of $\hat{U}^\top U$ as $\sigma_1\geq \cdots \geq \sigma_r \geq 0$. Then, the principal angles (see \cite{stewart_matrix_1990, chen_spectral_2021}) are defined as:
\[\Theta(U, \hat U) = \diag(\cos^{-1}(\sigma_1), \cdots, \cos^{-1}(\sigma_r))\] The $\sin\Theta$ distance measured by spectral norm is expressed as:
\begin{align*}
    \norm{\sin\Theta(U, \hat U)} = \norm{\hat U_\perp U} = \norm{U_\perp \hat U} = \norm{UU^\top - \hat U \hat U^\top}
\end{align*}

\begin{thm}[$\sin\Theta$ Theorem]\label{thm: robust_sin_rmtfa}
Suppose $L$ is the rank-$r$ symmetric matrix and $U\in\mathbb{O}^{p\times r}$ consists of eigenvectors of $L$. Let $\hat U$ consisted of the first $r$ eigenvectors of $\hat L$. On the event
\begin{align*}
   \mathcal{E}_\varrho& \defeq\left\lbrace 0< 3\norm{U}_{2, \infty}+ 2(1-\varrho)^{-1}\cdot \tfrac{\tau + \norm{\poffdiag(W)}}{\lambda_r^{(L)}} <\varrho < 1\right\rbrace,
\end{align*}
one has
\begin{align*}
    \norm{\sin\Theta(\hat U, U)} &\leq \tfrac{2}{1-\varrho}\cdot \tfrac{\tau + \norm{\poffdiag(W)}}{\lambda_r^{(L)}}
\end{align*}
\end{thm}

Theorem \ref{thm: robust_sin_rmtfa} establishes that, given the mild condition $\mathcal{E}_\varrho$, the discrepancy between the estimated subspace and the ground truth is bounded by $O\left(\frac{\tau + \norm{\poffdiag(W)}}{\lambda_r^{(L)}}\right)$. However, the leading constant is determined by the coefficient $\varrho\in(0, 1)$. For example, when $\varrho = \frac{1}{2}$, one has $\norm{\sin \Theta(\hat U, U)}\leq 4 \left(\tfrac{\tau+ \norm{\poffdiag(W)}}{\lambda_r^{(L)}}\right)$ on the event $$\mathcal E_\varrho=\left\lbrace \norm{U}_{2, \infty} + \tfrac{4}{3}\left(\tfrac{\tau+ \norm{\poffdiag(W)}}{\lambda_r^{(L)}}\right)< \tfrac{1}{6}\right\rbrace \supset \left\lbrace \norm{U}_{2, \infty} < \tfrac{1}{12}, \tfrac{\tau + \norm{\poffdiag(W)}}{\lambda_r^{(L)}} < \tfrac{1}{16}\right\rbrace$$
 Similar to the subset on the right, event of the form $\mathcal E^\dagger$ or $\mathcal E^\ddag$ in Remark \ref{remark: robust_sin_Theta_literature} is considered in the literature without focusing on the numerical constant. 


Since $\tau$ is a user-defined parameter, one viable approach for estimating $U$ involves  letting $\tau$ approach zero (equivalent to MTFA by Theorem \ref{thm:rmta_mtfa}) and considering the leading rank-$r$ subspace as the estimation. Recall from Proposition \ref{prop:mtfa_exact} and \ref{prop:balanced}, one can estimate $U$ perfectly if there is no off-diagonal noise ($\norm{\poffdiag(W)} = 0$) and the subspace is incoherent such that $\norm{U}_{2, \infty}< \frac{1}{\sqrt{2}}$. However, the Theorem suggested that when $\norm{\poffdiag(W)}>0$ and $\norm{U}_{2, \infty}$ large, it requires strong signal strength $\lambda_r^{(L)} \gg \norm{\poffdiag(W)}$ to guarantee subspace recovery due to the large leading constant.

We briefly discuss the proof of Theorem \ref{thm: robust_sin_rmtfa}. The full proof may be found in Appendix \ref{sec: robust_sin_rntfa_proof}. It has been recognized that the $\sin\Theta$ theorem can be established by bounding the errors at each step of the iterative algorithm (see Algorithm \ref{alg: rmtfa}), similar to HeteroPCA \cite{zhang_heteroskedastic_2022} (Algorithm \ref{alg: Heteropca}). Furthermore, thanks to the property of convergence to a unique fixed point (Theorem \ref{thm: convergence_beck}), regardless of the starting point, we can initiate the process from $D^{(0)} = \pdiag(\Sigma - L)$ where the information of the true $L$ is used, and consider the limit of the bound. Consequently, the estimation error bound is tighter at the starting point (since we use true information), leading to condition $\mathcal E_\varrho$ in Theorem \ref{thm: robust_sin_rmtfa} that is weaker than the conditions in HeteroPCA and Deflated-HeteroPCA (see $\mathcal E^\dagger$ and $\mathcal E^\ddag$ in the following Remark).

To observe how it addresses the Curse of Ill-condition presented in Section \ref{sec:contrib}, note that $\mathcal E_\varrho$ does not necessitate the inclusion of the condition on the condition number $\kappa$ of the matrix $L$, while HeteroPCA, represented by $\mathcal E^\dagger$, includes it. This enables us to tackle the challenge of ill-conditioning. Furthermore, given that $\mathcal{E}^\ddag\subset \mathcal E_\varrho$, it alleviates the necessity for stricter incoherence conditions imposed in Deflated HeteroPCA \cite{zhou_deflated_2023}. We further discuss the distinctions in the algorithms (HeteroPCA, Deflated HeteroPCA, and etc.) in Section \ref{sec: alg_connection}.

\begin{remark} \label{remark: robust_sin_Theta_literature}  Consider $\tau \lesssim \norm{\poffdiag(W)}$, $\mathcal{E}_\varrho$ is identified as the least stringent condition, as compared to Theorem 3 in \cite{zhang_heteroskedastic_2022} and  Theorem 4 in \cite{zhou_deflated_2023}, which establish the following inequality:
\begin{equation}
\label{eqn: robust_sin_Theta_literature}
    \norm{\sin \Theta(\hat U, U)}\lesssim  \tfrac{\norm{\poffdiag(W)}}{\lambda_r^{(L)}} + 2^{-t}    
\end{equation}
Here, $\hat U$ is the subspace estimator after running $t$ iterations of HeteroPCA (Algorithm \ref{alg: Heteropca}) and Deflated-HeteroPCA (Algorithm \ref{alg: deflated_Heteropca}), respectively. 

Let's define $\kappa$ as the condition number of matrix $L$, with $\kappa \defeq \frac{\lambda_1^{(L)}}{\lambda_r^{(L)}}\geq 1$. Theorem 3 in \cite{zhang_heteroskedastic_2022} demonstrates that \eqref{eqn: robust_sin_Theta_literature} holds on the event
$$\mathcal{E}^\dagger\defeq  \left\lbrace \sqrt\kappa \norm{U}_{2, \infty}< C_1, \tfrac{\norm{\poffdiag(W)}}{\lambda_r^{(L)}}<C_2\right\rbrace\subset \mathcal E_\varrho$$
for certain universal constants $C_1, C_2>0$. 
On the other hand, Theorem 4 in \cite{zhou_deflated_2023} establishes that \eqref{eqn: robust_sin_Theta_literature} holds on the event 
$$\mathcal{E}^\ddag \defeq \left\lbrace r \norm{U}_{2, \infty}< C_3, \tfrac{r\norm{\poffdiag(W)}}{\lambda_r^{(L)}}<C_4\right\rbrace\subset \mathcal E_\varrho$$
for certain universal constants $C_3, C_4 >0$.

This feature shows the robustness of our approach when dealing with ill-conditioned matrices characterized by a significant spectral gap or higher rank matrices.

\end{remark}

\section{Statistical Guarantees}
The outcome elucidated in Theorem \ref{thm: robust_sin_rmtfa} represents a deterministic assertion. In this section, we expound upon its applicability in various statistical contexts, highlighting the finite sample guarantees offered by our methodology.


\subsection{Factor model}

Let's consider $Y_1, \cdots, Y_n$ be iid samples generated by the model
\begin{align*}
    Y = X + Z \in \mathbb{R}^p
\end{align*}
where $\e(X) = \mu, \var(X) = L, \e(Z) = 0, \var(Z_i) = \omega_{i}^2, Z = (Z_1, \cdots, Z_p)^\top$, $X, Z_1, \cdots, Z_p$ are independent. In particular, $L$ is rank $r\leq\phi(p)$ and has spectral decomposition $U\Lambda U^\top$.

\begin{align*}
\Sigma \defeq \var(Y)= L+D,\\
D\defeq \diag(\omega_1^2,\cdots \omega_p^2). 
\end{align*}

Consider $(\hat L, \hat D)$ the solution to relaxed MTFA problem with regularization parameter $\tau>0$ given the sample covariance matrix 
$$\hat\Sigma = \frac{1}{n-1}\sum_{i=1}^n (Y_i - \bar Y)(Y_i -\bar Y)^\top, \quad \bar  Y = \frac{1}{n}\sum_{i=1}^n Y_i $$ as the input. For $\tau$ small enough, $\rank(\hat L)\geq r$ holds true because $\hat L\rightarrow \tilde L$ as $\tau\rightarrow 0 ^+$ and $\rank(\tilde L) > \phi(p) \geq r$ (Proposition \ref{prop:rank_reducibility}). Let $\omega_{\max{}}^2 \defeq \max_i\omega_i^2, \omega_{\text{sum}}^2 \defeq \sum_i \omega_i^2, \lambda_r^{(L)}:$ the $r$th largest eigenvalue of $L$. 
\begin{thm}\label{thm:factor_subspace}
Consider independent random variables $X$ and $Z_i$ are $L$-sub-Gaussian and $\omega_{i}^2$-sub-Gaussian respectively. Assume  $\norm{U}_{2, \infty}\leq c, n\geq C(r\vee \log (\lambda_r^{(L)}/\omega_{\text{sum}}^2))$ for some constant $c, C>0$. Then the first $r$ eigenvectors of $\hat L$ (solution to Relaxed MTFA \eqref{eqn:rmtfa}), $\hat U$, satisfies
\begin{align*}
    \e\norm{\sin\Theta(\hat U, U)} \lesssim  &
    \quad 
    \omega_{\max}\sqrt{\frac{\kappa (\tilde p \vee r)}{n \lambda_r^{(L)}}}+\sqrt{\frac{\tilde p}{n}}{ \frac{\omega_{\max}^2}{\lambda_r^{(L)}}}
\end{align*}
where $\tilde p = \omega_{\text{sum}}^2/\omega_{\max{}}^2, \kappa = \lambda_1^{(L)}/\lambda_r^{(L)}, \tau \lesssim n \lambda_r^{(L)}$.
\end{thm}

The proof is deferred to Appendix \ref{sec:factor_subspace_pf}

\begin{remark}
The outcome aligns with the main result Theorem 1 presented in \cite{zhang_heteroskedastic_2022}. The findings in \cite{zhang_heteroskedastic_2022} were established by executing their algorithm over a sufficiently large number of iterations, whereas we provide a direct characterization of the optimization problem \eqref{eqn:rmtfa} minimizer's behavior and introduce an converging algorithm. Notably, our deterministic analysis (Theorem \ref{thm: robust_sin_rmtfa}) operates without the need for the bounded condition number assumption made in \cite{zhang_heteroskedastic_2022}.  Therefore, it is possible to obtain a condition number free bound based on  Theorem \ref{thm: svd_hetero}, as long as one assume the signal strength $\lambda_r^{(L)}$ is large enough.
\end{remark}

\begin{remark}
    
The estimator achieve the following optimal rate \cite{zhang_heteroskedastic_2022}\begin{align*}
     \inf _{\hat U} \sup _{\Sigma } \e\,\norm{\sin\Theta(\hat U, U)} \asymp \frac{1}{\sqrt{n}} \left(\frac{\tilde\sigma+r^{1/2}\sigma_*}{\lambda^{1/2} }+ \frac{\tilde\sigma \sigma_*}{\lambda}\right). 
 \end{align*}
 Here, the infimum is taken over all possible subspace estimators, and the true $\Sigma$ is chosen adversarially from the set:
 \begin{align*}
     \left\lbrace \Sigma = L+D: D = \pdiag (D), \tr(D)\leq \tilde{\sigma}^2, \norm{D} \leq \sigma_*^2, \norm{U}_{2, \infty} \leq 1/12, \lambda_r^{(L)} \geq \lambda, \lambda_1^{(L)}/\lambda_r^{(L)}\leq \kappa) \right\rbrace
 \end{align*}
 parametrized by postive numbers $(\tilde\sigma, \sigma_*, \lambda, \kappa)$ satisfying $\sqrt{p}\sigma_* \geq \tilde\sigma\geq \sigma_* >0, \kappa \geq 1$. 
 The result combines the upper bound (Theorem \ref{thm:factor_subspace}) and the lower bound constructed in \cite{zhang_heteroskedastic_2022}. 
\end{remark}

 
As factor analysis finds widespread practical application,  we address a common issue known as Heywood cases. These issues arise when factor extraction methods, such as Maximum Likelihood or MTFA, produce a decomposition $\hat L + \hat D$, where $\hat L \in \mathbb{S}^p_+$ is of low rank, and $\hat D$ is diagonal, with $\hat D_{i,i} \leq 0$ for certain indices $i$. This is problematic because $D_{i,i} = \omega_{i}^2$ corresponds to the variance of $Z_i$. Various rules of thumb have been proposed to mitigate this issue, but consensus on the best approach is lacking. So we present the following lemma:

\begin{lmm}\label{lmm:heywood}
    For a sufficiently large value of $\tau$, the diagonal elements $\hat D_{i,i}$ become positive for all $i$.
\end{lmm}

Lemma \ref{lmm:heywood} implies that Heywood cases can be consistently avoided by choosing a larger regularization parameter $\tau$. When $\tau$ is small, the factor $\hat L$ is extracted under the assumption that the off-diagonal elements are accurate, whereas increasing $\tau$ corresponds to a more substantial shrinkage to account for the noise inherent in the input covariance matrix $\Sigma$. Consequently, there is a bias and variance trade-off exists in this context. In practice, the choice of the tuning parameter can be determined through cross-validation.

\subsection{SVD under heteroskedastic noise} \label{sec: svd_hetero_setting}

We present the result considering the following general assumption on the data matrix
\begin{align}\label{eqn: matrix_model}
    Y = M + Z \in \mathbb{R}^{p\times n}
\end{align}
where $Z$ is a zero-mean noise matrix and $M$ is a rank-$r$ matrix with singular decomposition $M = U\Lambda V^\top = \sum_{i=1}^r \sigma_i u_iv_i^\top$. This general model is root to many high dimensional statistics problems such as  bipartite stochastic block models, exponential family PCA \cite{liu_epca_2018}, Poisson PCA \cite{salmon_poisson_2012}, and Tensor PCA \cite{zhang_tensor_2018}.  

We are able to match the main result, Theorem 1 in \cite{zhou_deflated_2023}, as well as Theorem 4 in \cite{zhang_heteroskedastic_2022}. The former is robust to ill-conditioned matrix $M$, and our result Theorem \ref{thm: svd_hetero} further reduce the signal strength condition into ``minimal'' up to logarithm factor. See Remark \ref{remark:signal_strength} on \eqref{svd_hetero:signal}. We make sub-Gaussian assumption in Assumption \ref{assmpt: error_term} for simplicity, but it can be as general as Assumption 1 in \cite{zhou_deflated_2023}.

\begin{assmpt}
    \label{assmpt: error_term}
\begin{itemize}
    \item $Z_{i, j}$ are independent random variables with mean $\e (Z_{i, j}) = 0$ for all $i, j$. 
    \item With $\omega_{i, j}^2 \defeq \var(Z_{i, j})$ allowed to be location dependent, hence account for heteroskedastic noise, 
 we define\begin{align*}
     \omega_{\max}^2 \defeq \max_{i, j}\omega_{i, j}^2,\; \omega_{\row}^2 \defeq \max_{i} \sum_{j=1}^n \omega_{i, j}^2,\; \omega_{\col}^2 \defeq \max_{j} \sum_{i=1}^p \omega_{i, j}^2
\end{align*}
$Z_{i, j}$ is $\omega_{i, j}^2$-sub-Gaussian. 
\end{itemize}
\end{assmpt}

\begin{thm}\label{thm: svd_hetero} 
There are universal constants $C_0, C>0$ such that if 
\begin{subequations}
\begin{align}
        &\max(\norm{U}_{2, \infty}, \norm{V}_{2, \infty })< C\label{svd_hetero:incoherence}\\
        &\sigma_r\geq C_0 (\omega_{\col} + \sqrt{\omega_{\col}\omega_{\row}}) \sqrt{\log (n\vee p)}\label{svd_hetero:signal}\\
        &\omega_{\col}^2/\omega_{\max}^2 \gtrsim \max(p\norm{U}_{2, \infty}^2, n\norm{V}_{2, \infty}^2)\label{svd_hetero:var_spike}
\end{align}
\end{subequations}
 the leading $r$ eigenvectors $\hat U$ of $\hat L$ (solution to \eqref{eqn:rmtfa}) with input covariance matrix $\Sigma = YY^\top $ satisfies
        \begin{align*}
            \e\norm{\sin\Theta(\hat U, U)} & \lesssim 
            \frac{\omega_{\col} \sqrt{\log (n\vee p)}}{\sigma_r} + \frac{\omega_{\row} \omega_{\col}{\log (n\vee p)}}{\sigma_r^2}
        \end{align*}
    \end{thm}

The proof can be found in Appendix \ref{sec: svd_hetero_pf}. 

\begin{remark}\label{remark:signal_strength}
    For homoskedastic noise, $\omega_{i, j}\asymp\omega_{\max}$, the result matches minimax rate \cite{cai_rate-optimal_2018,cai_subspace_2021} up to a log factor.
    \eqref{svd_hetero:signal} is less stringent than \cite{zhou_deflated_2023} by a factor of $r$. Furthermore, the condition
    $$\sigma_r \gtrsim [(np)^{1/4}+p^{1/2}]\log (n\vee p)$$
    matches, up to a logarithmic factor, the necessary condition for the existence of a consistent estimator in subspace estimation, as suggested in Theorem 3.3 in \cite{cai_subspace_2021}.
\end{remark}

\subsection{Subspace Estimation with Missing Values}
We present the subspace estimation result in the presence of missing values and heteroskedastic noise, without relying on condition number assumptions. 

It is worth noting that existing results related to missing values inference in the presence of  heteroskedastic noise hinge on condition number considerations, as seen in \cite{yan_inference_2021,agterberg_entrywise_2022}. Addressing the challenge posed by ill-conditioning in inference remains a promising avenue for future research.

We consider a general mathematical model where we aim to separate a signal $M$ (identifying its left-eigenspace) from noise $Z$ using the data matrix $Y$:
\begin{align}
    Y = \proj{\Omega_\text{obs}}(M+Z)    \in\mathbb{R}^{p\times n}
\end{align}
Only the elements with index $(i, j) \in \Omega_\text{obs}\subseteq [p]\times [n]$ are observed, meaning that
\begin{align}
    Y_{i, j} &= \begin{cases}M_{i, j} + Z_{i, j}, & \text{for }(i, j) \in \Omega_\text{obs}\\
    0, & \text{for }(i, j) \not\in \Omega_\text{obs}
    \end{cases}
\end{align}
Assume $Y_{i, j}$ is missing at random with probability $\theta\in (0, 1)$.

\begin{thm}
Assume incoherence condition $\norm{U}_{2,\infty}< c$ for some constants $c>0$, $\max_{i, j}\norm{Y_{i, j}}_{\psi_2} \leq C$. The leading $r$ eigenvectors $\hat U$ of $\hat L$ (solution to \eqref{eqn:rmtfa}) with input $\Sigma = YY^\top $ satisfies
\begin{align*}
    \norm{\sin\Theta(\hat U, U)} \lesssim \frac{\max (\sqrt{n(\theta+\theta^3 n^2)\log n}, \theta p \log^2(p))}{\theta^2 {(\lambda_r^{(M)})}^2 } 
\end{align*}
with probability at least $1- p^{-C}$
\end{thm}

\section{Proposed Algorithm}
\label{sec:alg}

In this section, we present Algorithm \ref{alg: rmtfa} as a numerical solution for Relaxed MTFA \eqref{eqn:rmtfa}. We introduce the eigenvalue soft-thresholding operator employed in the algorithm and establish the sub-linear convergence property in the worst case of the iterative process.  Consequently, the fixed-point equation characterizes the unique solution to \eqref{eqn:rmtfa}. Additionally, our main theorem (Theorem \ref{thm: robust_sin_rmtfa}) is derived by analyzing the estimation error in each iteration of Algorithm \ref{alg: rmtfa}.

\subsection{Eigenvalue soft-thresholding}
We first introduce the essential ingredient of our  algorithm. Given a symmetric matrix $M\in \mathbb{S}^p$ with spectral decomposition:
\[M =\sum_{i=1}^p \lambda_i u_iu_i^\top\]
where $\lambda_1\geq \cdots \geq \lambda_p$, unit length $u_i\in\mathbb{R}^p$ for $i=1, \cdots, p$, define the eigenvalue soft-thresholding operator by
\[D_\tau^+(M) \defeq \sum_{i=1}^p (\lambda_i - \tau)_+ u_iu_i^\top\]
where $(\cdot)_+\defeq \max(\cdot, 0)$ is taking the positive part of the truncated eigenvalues. It is clear that $D_\tau^+(M)$ is a positive positive semi-definite matrix since it is symmetric and all its eigenvalues are non-negative.  In fact, we have the following lemma.

\begin{lmm}\label{lmm: spectral_thresholding}
$D_\tau^+(M)$ is the unique minimizer of the nuclear norm optimization problem with positive semi-definite constraint:
\begin{align*}
    \minimize_{X\in \mathbb{S}^p_+} \tau \norm{X}_* + \frac{1}{2} \norm{X-M}_F^2.
\end{align*}
That is,
\begin{align*}
    \mathcal{D}_\tau^+(M)= \argmin_{X\in \mathbb{S}^p_+} \tau \norm{X}_* + \frac{1}{2} \norm{ X -  M }_F^2.
\end{align*}
\end{lmm}
It has been proven by \cite{lounici_high-dimensional_2014}
using the characterization of nuclear norm sub-gradient, while we give a direct prove by Von-Neuman trace inequality in the Appendix \ref{subsection: proof_spectral_thresholding}. It is worth noting that without the positive semi-definite constraint, the solution corresponds to singalur value thresholding \cite{cai_singular_2010} and the iterative matrix completion algorithm -- Soft-Impute \cite{mazumder_spectral_2010} is based on that. We will explore more in Section \ref{sec: alg_connection}. For now, let's define similarly the singular value soft-thresolding operator
$$\mathcal{D}_\tau(M) \defeq \sum_{i=1}^p \text{sign}(\lambda_i)(|\lambda_i|-\tau)_+ u_iu_i^\top $$
for symmetric matrix $M$.

\subsection{Alternating Minimization Algorithm and Convergence}

    \begin{algorithm}
            
	      	\caption{Relaxed MTFA Algorithm \label{alg: rmtfa}}
	      	\begin{algorithmic}
	      		\Require $\tau>0$, input covariance matrix ${\Sigma}$, diagonal matrix $D^{(0)}$
	      		\For{$k = 1, 2, \cdots$} 
	      		\State $L^{(k)} = \mathcal{D}_\tau^+ ( \Sigma - D^{(k-1)})$ 
                        \State $D^{(k)} = \pdiag({\Sigma} - L^{(k)})$
	        		\EndFor

                \end{algorithmic}
	      \end{algorithm}

We now present our algorithm for solving the convex optimization problem. Algorithm \ref{alg: rmtfa} constructs a sequence of variables $(L^{(k)}, D^{(k)})$ by iteratively minimizing $\mathscr L$ with fixed $\mathscr D$ and then minimizing $\mathscr D$ with fixed $\mathscr L$. This decomposition leads to explicit closed-form solutions. Specifically, according to Lemma \ref{lmm: spectral_thresholding}:

\begin{align*}
L^{(k)} &= \argmin_{\mathscr L\in \mathbb{S}^p_+} F(\mathscr L, D^{(k-1)})= \mathcal{D}^+_\tau(\Sigma-D^{(k-1)})\\
D^{(k)} &= \argmin_{\mathscr D = \pdiag(\mathscr D)} F(L^{(k)}, \mathscr D)= \pdiag( \Sigma - L^{(k)})
\end{align*}
As a result, the objective function value monotonically decreases:
\begin{align*}
F(L^{(1)}, D^{(0)})\geq F(L^{(1)}, D^{(1)}) \geq F(L^{(2)}, D^{(1)})\geq
F( L^{(2)}, D^{(2)}) \geq \cdots
\end{align*}
until a stationary point is reached. Notably, there are no additional parameters, such as step size, that need to be chosen.

\begin{thm}\label{thm:fixedpt}
$(L, D)$ is the solution to \eqref{eqn:rmtfa} if and only if it
is a fixed point of the continuous mapping:\begin{align*}\label{eqn:fixedptmap}
\begin{pmatrix}
        L\\D
    \end{pmatrix}  \mapsto \begin{pmatrix}
        \mathcal{D}^+_\tau(\Sigma - D)\\
        \pdiag(\Sigma - L)
    \end{pmatrix} 
\end{align*}

\end{thm}

\begin{remark}
    In terms of variable $L$ alone, the fixed point equation can be expressed as \begin{align*}
    L = \mathcal D_\tau^+ (\poffdiag(\Sigma) + \pdiag(L)))
\end{align*} 
the equation has a unique solution for all $\tau > 0$.  This provides a simple characterization of the solution to \eqref{eqn:rmtfa}.
\end{remark}

\begin{remark} The update rule can be alternatively expressed as:
\begin{align*}
    L^{(k)} = \mathcal{D}^+_\tau(\poffdiag(\Sigma )+\pdiag(L^{(k-1)}))    
\end{align*}
This update rule bears resemblance to the Soft-Impute method \cite{mazumder_spectral_2010}, especially in cases where the diagonal entries are missing. However, it's important to note that our algorithm differs from theirs because we project the matrix to ensure positive semi-definiteness in each iteration. 
\end{remark}

We can establish the sub-linear convergence of the objective function and the convergence of the iterates to the global minimizer, as described in Theorem \ref{thm: convergence_beck}.

    \begin{thm}\label{thm: convergence_beck}
 Given $x_0 = (L^{(0)},D^{(0)})$ with any $L^{(0)}\in \mathbb{R}^{p\times p}$ in the feasible region and $D^{(0)}=\pdiag(\Sigma - L^{(0)})$. Let $x_{k} = (L^{(k)}, D^{(k)})$ for $k\in\mathbb{N}$ be the updated variables in Algorithm \ref{alg: rmtfa}  and $x^*$ is the unique minimizer of the objective function $F$.
        \begin{itemize}
            \item The iterates $\{x_k\}$ converges to $x^*$
            \item  We have the following convergence rate
        \begin{align}
            F(x_{k+1}) - F(x^*)\leq \max\left\lbrace \frac{F(x_0)-F(x^*)}{2^{k}},\frac{16R^2}{k} \right\rbrace 
        \end{align}
        for all $k\geq 1$  where $R =\max\{\norm{\vectorize((L, D)-x^*)}: F(L, D)\leq F(x_0), L \in\mathbb{S}^p_+, D = \pdiag(D)\}$
        \end{itemize}  
    \end{thm}

The inequality presented in Theorem \ref{thm: convergence_beck} also highlights the benefits of employing a warm start --- initializing the algorithm with values that are already in proximity to the optimal solution, often acquired from a prior run or through a heuristic approach.

The quantity $R$ is contingent on how closely the initial value deviates from the global minimum. A small value of $R$ can significantly expedite the convergence process towards an optimal or near-optimal solution. This is attributed to the fact that the reduction of the gap between the function values and the global minimum follows an exponential decay until it approaches the sub-linear term. The proof is based on \cite{beck_convergence_2015} and the conditions are checked in \ref{subsection: relaxed_mtfa_convergence}. 

\section{Connections to Literature}

In this section, we present the connections of our method to
existing works. Our proposed method may be viewed as a LASSO \cite{tibshirani_regression_1996} for factor analysis, which is closely related to \cite{mazumder_spectral_2010, lounici_high-dimensional_2014}. In addition, we establish connections between HeteroPCA and Principle Axis Method in factor analysis, and our algorithm, which is a key observation that allow us to carry out the theoretical analyses of Theorem \ref{thm: robust_sin_rmtfa}.

\subsection{LASSO for Factor Analysis}
When it comes to regression analysis, LASSO \cite{tibshirani_regression_1996} provides a practical solution to the computationally intractable NP-hard problem of best subset selection. This enables the identification of a parsimonious linear regression model. In the context of factor analysis, non-convexity and NP-hardness inherent in the rank-constrained problem exemplified by  Minimum Residual Factor Analysis (MINRES) \cite{harman_factor_1966,shapiro_statistical_2002}, also known as Principal Axis Factoring (PAF) \cite{de_winter_factor_2012}
\begin{align}\label{eqn: fa_minres}
\begin{aligned}
    \minimize \quad& \tfrac{1}{2}\norm{\Sigma - (\mathscr L + \mathscr D)}_F^2\\
    \text{s.t.} \quad& \rank(\mathscr L) \leq r,\\
    &\mathscr L \in \mathbb{S}^p,\\
    &\mathscr D = \pdiag (\mathscr D).
\end{aligned}
\end{align}
among many other generalizations \cite{bertsimas_certifiably_2017}.  
 Our proposed method \eqref{eqn:rmtfa} can be viewed as relaxing the rank constraint in \eqref{eqn: fa_minres} by nuclear norm regularization. In addition, 
 \eqref{eqn:rmtfa} is minimizing the least square objective function $\norm{\Sigma  - (\mathscr L + \mathscr D)}_F^2$  while regularizing the $ l_1$-norm (LASSO penalty) of eigenvalues $\lambda^{(\mathscr L)}=(\lambda_1^{(\mathscr L)}, \lambda_2^{(\mathscr L)} \cdots, \lambda_{p}^{(\mathscr L)}) \in \mathbb{R}^p$ of symmetric matrix $\mathscr L$:
\begin{align}
    \norm{\mathscr L}_* = \sum_{i=1}^p |\lambda_i^{(\mathscr L)}| = \norm{\lambda^{(\mathscr L)}}_1
\end{align}
So in the same vein, the optimization problem \eqref{eqn:rmtfa} avoids the NP-hardness by convex relaxation. 

While LASSO performs variable selection and fitting simultaneously, Relaxed MTFA operates in a similar fashion. It seeks the correct subspace (variable selection) and corresponding eigenvalues (fitting). We examine some properties shared between LASSO and \eqref{eqn:rmtfa}. Firstly, under orthogonal design, the LASSO estimator (given regularization parameter $\tau>0$) can be simplified as soft-thresholding operator $\text{sign}(\cdot)(|\cdot|-\tau)_+$ applied to the ordinary least square estimator. The same is true for Relaxed MTFA: let $$\hat L = \sum_{i=1}^r \lambda_i^{(\hat L)} u_iu_i^\top$$
be the solution to \eqref{eqn:rmtfa} given input $\Sigma$ and tuning parameter $\tau>0$. Then according to Theorem \ref{thm:fixedpt} and basic algebra, the magnitude $\lambda_i^{(\hat L)}$ for an identified subspace $u_i$ is equal to the soft-thresholding operator applied to ordinary least square coefficient:
$$\argmin_\alpha\quad \norm{\Sigma -\hat D- \alpha u_iu_i^\top}_F^2. $$

Furthermore, when the regularization parameter $\tau$ is set to zero, the optimization problem \eqref{eqn:rmtfa} admits infinitely many solutions with different nuclear norms $\norm{\mathscr L}_*$--a scenario analogous to linear regression when the number of explanatory variables $p$ exceeds the number of samples $n$. To ensure unique regression coefficients in LASSO, a positive tuning parameter is necessary \cite{tibshirani_lasso_2013}. Similarly, in this work, we consider $\tau > 0$ to guarantee a unique solution through regularization. 

Finally, akin to LASSO, where the degree of sparsity in the solution is determined by the penalty parameter's magnitude, Lemma \ref{lmm: spectral_thresholding} demonstrates that our proposed method follows a similar pattern. Specifically, a higher value of $\tau$ corresponds to a lower rank of the solution with respect to the matrix variable $\mathscr L$. Therefore, selecting an appropriate $\tau$ value can assist in achieving a specific rank.

We would like to comment on the matrix LASSO estimator proposed in \cite{lounici_high-dimensional_2014}. Their method is designed to estimate the covariance matrix in the presence of missing values. It applies the eigenvalue soft-thresholding operator to a matrix $\Sigma$ with its diagonal rescaled. The underlying concept involves correcting the diagonal first, followed by denoising through $\mathcal D_\tau^+$.  Notably, the distinction between \cite{lounici_high-dimensional_2014} and Relaxed MTFA is that the latter does not directly leverage information from diagonal cells and instead imputes values iteratively. A theoretical comparison between the two methods is a subject for future investigation.

\subsection{General Framework to handle Heteroskedastic noise}\label{sec: alg_connection}

To establish the connection between Relaxed MTFA and existing methods such as HeteroPCA, Deflated-HeteroPCA, Factor Analysis, and Soft-Impute, let's consider the following penalized optimization problem:

\begin{align}\label{eqn: general_opt}
\begin{aligned}
    \minimize \quad& \Pi(\mathscr L)+\tfrac{1}{2}\norm{\Sigma - (\mathscr L + \mathscr D)}_F^2\\
    \text{s.t.} \quad& \mathscr L \in \mathbb{S}^p,\\
    &\mathscr D = \pdiag (\mathscr D).
\end{aligned}
\end{align}
Here, $\Pi$ represents a penalty function utilized to regularize the component $\mathscr L$, encouraging it to be low rank. Regardless of convexity, achieving a stationary point of \eqref{eqn: general_opt} is possible through the application of Alternating Minimization. This is attributed to the monotone decreasing objective function value in each iteration. Algorithm \ref{alg: general} outlines the Alternating Algorithm, which incorporates the proximal map
$$\prox{}_\Pi(M) \defeq\quad  \argmin_{\mathscr L\in \mathbb{S}^p} \quad \Pi(\mathscr L) + \tfrac{1}{2}\norm{M - \mathscr L}_F^2. $$

\begin{algorithm}
            
	      	\caption{Alternating Minimization Algorithm for solving \eqref{eqn: general_opt} \label{alg: general}}
	      	\begin{algorithmic}
	      		\Require 
                    input covariance matrix $\Sigma$, diagonal matrix $D^{(0)}$
	      		\For{$k = 1, 2, \cdots$} 
	      		\State $L^{(k)} = \prox_\Pi(\Sigma - D^{(k-1)})$
                        \State $D^{(k)} = \pdiag({\Sigma} - L^{(k)})$
	        		\EndFor
	      	\end{algorithmic}
	      \end{algorithm}
Obviously, when  $$\Pi(M) = \tau \norm{M}_*+ \infty \cdot 1_{\{M \not\in \mathbb{S}^p_+\}}$$
\eqref{eqn: general_opt} is equivalent to Relaxed MTFA \eqref{eqn:rmtfa}, and $\prox{}_\Pi = \mathcal D_\tau^+$ as we have considered the case in Section \ref{sec:alg}. 
We will establish connections between different algorithms by examining various choices of $\Pi$. 

{\it Factor Analysis and HeteroPCA.} Consider
    $$\Pi(M) = \infty\cdot 1_{\{\rank(M)>r\}}$$ 
    In this case, non-convex optimization \eqref{eqn: general_opt} is equivalent to the Factor Analysis problem \eqref{eqn: fa_minres}. As a popular fitting approach alongside maximum likelihood estimation \cite{de_winter_factor_2012}, Algorithm \ref{alg: general} with
    $$\prox{}_\Pi (M) = \argmin_{\rank(\mathscr L)\leq r} \norm{M - \mathscr L}_F^2\quad\text{(rank-$r$ SVD)}$$
    is implemented in Statistical software such as SPSS and R package \textit{psych}, with different heuristics in initialization of $D^{(0)}$ \cite{grieder_algorithmic_2021}. Such algorithm is called principal axis method or principal factor method in \cite{bartholomew_latent_2011}.

    On the other hand, HeteroPCA (Algorithm \ref{alg: Heteropca}) is an iterative procedure that commences by setting the diagonal entries of the Gram matrix $\Sigma$ to zero and subsequently updating them via low-rank approximation. Surprisingly, the following observation (Lemma \ref{lmm: minres_equiv}) has not been explicitly stated in the literature, even though HeteroPCA is interpreted as the projected gradient descent of the optimization problem
    \begin{equation}\label{eqn: minres_equiv}
        \minimize_{\rank(\mathscr L)\leq r}\quad\norm{\poffdiag(\Sigma- \mathscr L)}_F^2
    \end{equation}
    in \cite{zhang_heteroskedastic_2022}.  
    \begin{lmm}[Equivalance of Factor Analysis and HeteroPCA]\label{lmm: minres_equiv}
        In terms of the optimization problem,
        \eqref{eqn: minres_equiv} is equivalent to \eqref{eqn: fa_minres} for $\Sigma \in \mathbb{S}_+^p$. Moreover, Principal axis method (Algorithm \ref{alg: general} with $\Pi(M) = \infty \cdot 1(\rank(M)>r)$) is equivalent to HeteroPCA (Algorithm \ref{alg: Heteropca}): Given $r, \Sigma$ and $G^{(0)} + D^{(0)} =\Sigma$, the variable $L^{(k)}$ produced by both algorithm are exactly the same for all $k$.
    \end{lmm}

{\it Soft-Impute \cite{mazumder_spectral_2010}.} Consider $$\Pi(M) = \tau \norm{M}_*$$
    then \eqref{eqn: general_opt} is equivalent to
    \begin{equation}\label{eqn:soft_impute_equiv}
      \minimize_{\mathscr L\in\mathbb{R}^{p\times p}} \quad \tau\norm{\mathscr L}_* + \frac{1}{2}\norm{\poffdiag(\Sigma - \mathscr L)}_F^2  
    \end{equation}
    for symmetric $\Sigma$. This optimization problem \eqref{eqn:soft_impute_equiv}, is known as Soft-Impute, as introduced by Mazumder et al. \cite{mazumder_spectral_2010}. In the context of the matrix completion problem, the goal is to impute the missing values with the observed entries using the nuclear norm minimization paradigm. The convex optimization problem in  \eqref{eqn:soft_impute_equiv} aims to impute the diagonal elements of the matrix. The proximal map for this problem corresponds to singular value thresholding, i.e. $\prox{}_\Pi = \mathcal D_\tau$, the corresponding Alternatiing Minimzation Algorithm has been shown convergent.

Theoretical analysis for noisy matrix completion with SVD \cite{abbe_entrywise_2020} and Soft-Impute \cite{chen_noisy_2020} are established, there are, to our knowledge, no condition number free statistical recovery guarantee. However, due to the algorithm's close resemblance to HeteroPCA and Algorithm \ref{alg: rmtfa}, similar techniques can be employed to derive a subspace recovery theorem. We defer this exploration to future work.

Upon closer examination of the optimization problem \eqref{eqn:soft_impute_equiv}, it becomes evident that it does not impose the constraint of positive semi-definiteness on variable $\mathscr L$. This distinction from \eqref{eqn:rmtfa} could pose a challenge for factor analysis applications, as covariance matrices cannot have negative eigenvalues. This explains why Soft-Impute occasionally exhibits inferior performance in numerical simulations, as observed in Section \ref{sec:numerical}.

{\it HeteroPCA with positive semidefinite constraint.} Motivated by empirical improvement with positive semi-definite constraint (Soft-Impute vs. Relaxed MTFA), we consider additional constraint
\[\Pi(M) = \infty \cdot 1_{\{\rank(M)>r \text{ or } M\not\in\mathbb{S}^p_+\}}\]
for HeteroPCA algorithm, which corresponds to Algorithm \ref{alg: Heteropca+}. A $\sin\Theta$ theorem can similarly be established for this algorithm. While the curse of ill-conditioning may still persist due to the sensitivity of initialization to the converging stationary point, empirical results in Section \ref{sec:numerical} demonstrate its robustness under various configurations. The estimator produced by Algorithm \ref{alg: Heteropca+} can be employed as an initialization for Relaxed MTFA \eqref{alg: rmtfa} to achieve faster convergence.

\begin{algorithm}

\caption{HeteroPCA with positive semi-definite constraint \label{alg: Heteropca+}}
	      	\begin{algorithmic}
	      		\Require 
                    rank $r\in \mathbb N$, input covariance matrix $\Sigma$, diagonal matrix $D^{(0)} = \pdiag(\Sigma)$
	      		\For{$k = 1, 2, \cdots, T_{\max{} }$} 
	      		\State $L^{(k)} = \minimize_{\rank(\mathscr L)\leq r, \mathscr L \in \mathbb{S}^p_+}\norm{(\Sigma - D^{(k-1)})-\mathscr L}_F^2$
                        \State $D^{(k)} = \pdiag({\Sigma} - L^{(k)})$
	        		\EndFor
	      	\end{algorithmic}
	      \end{algorithm}

{\it Deflated HeteroPCA.} Algorithm \ref{alg: deflated_Heteropca} can be viewed as an instance of Algorithm \ref{alg: general} with a proximal map that depends on the iteration number $k$ and the singular values of its argument.  The core concept behind Deflated HeteroPCA is that, instead of extracting an $r$-dimensional subspace simultaneously, it incrementally increases the dimension to extract until the desired dimension of $r$ is reached. This approach is employed to avoid jeopardizing the condition-number-free error bound. In a sense, this approach resembles Soft-Impute and Relaxed MTFA (Algorithm \ref{alg: rmtfa}), where solving the problem with a higher tuning parameter first and then gradually decreasing it to the desired $\tau$ yields faster convergence due to warm starting. (Theorem \ref{thm: convergence_beck})

In summary, Algorithm \ref{alg: general} serves as a template for estimating subspaces in the presence of heteroscedastic noise. Alternative choices for $\Pi$ that were not explored here include the weighted nuclear norm, truncated nuclear norm, and regularization of the rank function. We defer further exploration of these techniques to future work.

\section{Numerical}\label{sec:numerical}

Consider the simple statistical model in Section \ref{sec: svd_hetero_setting}, we are going to compare the aforementioned methods for various parameters. 

\subsection{Algorithms, Data Generation and Expectations}
 We will perform numerical experiments through simulations with $50$ samples, comparing our proposed methods—specifically, relaxed MTFA (Algorithm \ref{alg: rmtfa}) --- against baseline algorithms, including SVD (PCA), Diagonal-Deleted PCA (Algorithm \ref{alg: diag_deleted_pca}), HeteroPCA (Algorithm \ref{alg: Heteropca}), Deflated HeteroPCA (Algorithm \ref{alg: deflated_Heteropca}), HeteroPCA with a positive semi-definite constraint (Algorithm \ref{alg: Heteropca+}), and Soft-Impute. To simplify, we will use the acronyms as summarized in Table \ref{tab:acronym_and_method_summary}.

\begin{table}[H]
    \centering
    \begin{tabular}{llll}
         Acronym & Method & Algorithm   \\\hline
         SVD& Singular value decomposition \\
         DD& Diagonal-deleted PCA & \ref{alg: diag_deleted_pca}\\
         HPCA& HeteroPCA & \ref{alg: Heteropca}\\
         DHPCA& Deflated HeteroPCA & \ref{alg: deflated_Heteropca}\\
         HPCA+ & HeteroPCA with PSD constraint & \ref{alg: Heteropca+}\\
         rMTFA & Relaxed MTFA (our proposal) & \ref{alg: rmtfa}\\
         SI & Soft-Impute &\ref{alg: general} with $\prox{}_\Pi = \mathcal{D}_\tau$
    \end{tabular}
    \caption{Algorithms summary and acronym for numerical experiments}
    \label{tab:acronym_and_method_summary}
\end{table}

\begin{table}[H]
    \centering
    \begin{tabular}{ll}
        Parameter  & Description  \\\hline
        $n, p\in\mathbb N$& dimension of the data matrix $Y= M+Z\in\mathbb{R}^{p\times n}$ \\
        $r\in\mathbb N$ & the rank of signal matrix $M$\\
        $\kappa \geq 1$& the condition number of the signal matrix\\
        $\omega>0$ & maximum noise level of noise matrix $Z$
         
    \end{tabular}
    \caption{Parameter summary for numerical experiments}
    \label{tab:param_summary}
\end{table}

The parameters $(n, p, r, \kappa, \omega)$ for generating the statistical model are summarized in Table \ref{tab:param_summary}. With the parameters  are provided, we construct the signal matrix $M$ by the following steps:
\begin{itemize}
    \item Simulate $U\in\mathbb{R}^{p\times r}$ and $V\in\mathbb{R}^{r\times n}$ by the leading $r$ left/right singular vector of $p\times n$ random matrix -- each entry is iid standard Gaussian. 
    \item Motivated by \eqref{svd_hetero:signal}, we set the signal strength 
    $\sigma_r = (np)^{1/4}+p^{1/2}, \sigma_{r-i} = \kappa^{i/(r-1)} \cdot \sigma_r$ for $i = 1, \cdots, r-1$.
    \item Construct the signal matrix $M = U\diag(\sigma_1, \cdots, \sigma_r) V^\top$. 
\end{itemize}
For the noise matrix $Z$, we generate it by the following step:
\begin{itemize}
    \item Simulate row-wise noise standard deviation $\{\omega_i\}_{i=1}^p$ by uniform random variables on $[0, \omega]$. 
    \item Generate $Z = \diag(\omega_1, \cdots, \omega_p) Z_0$ where $p\times n$ matrix $Z_0$ has iid standard normal entries. 
\end{itemize}
Given $Y = M+Z$, we feed $\Sigma = YY^\top$ into each method, obtaining the leading $r$ eigenspace as estimate, then compute the average $\sin\Theta$ distance.

Here, we describe the details of the implementation. We let the number of iteration $T_{\max{}} = T_{\max{}}^{(1)}=\cdots = 30$ in HeteroPCA based algorithm (HPCA, HPCA+, DHPCA). For Relaxed MTFA and Soft-Impute, as the estimator is not sensitive to the regularization parameter $\tau$ when it is selected small enough, for simplicity we take
$$\tau = \sigma_{r}^2/16.$$

In our numerical experiments, we fix the parameters at $(n, p, r, \kappa, \omega) = (200, 50, 5, 3, 1)$ and systematically vary one parameter at a time. It is anticipated that standard SVD may struggle in the presence of heteroskedastic noise. Referring to the recovery conditions outlined in Remark \ref{remark: robust_sin_Theta_literature}, HPCA might face challenges with large condition numbers, whereas DHPCA could encounter difficulties with increasing values of $r$. By introducing positive semi-definite constraints and convexity to HPCA, HPCA+ and SI are expected to exhibit increased robustness. Overall, rMTFA is anticipated to remain competitive across various scenarios. Nevertheless, all methods may experience reduced performance in the presence of high levels of noise $(\omega)$.

\subsection{Results and discussion}

\begin{figure} [H]
\centering
\begin{tabular}{ccc}
\includegraphics[width=0.45\textwidth]{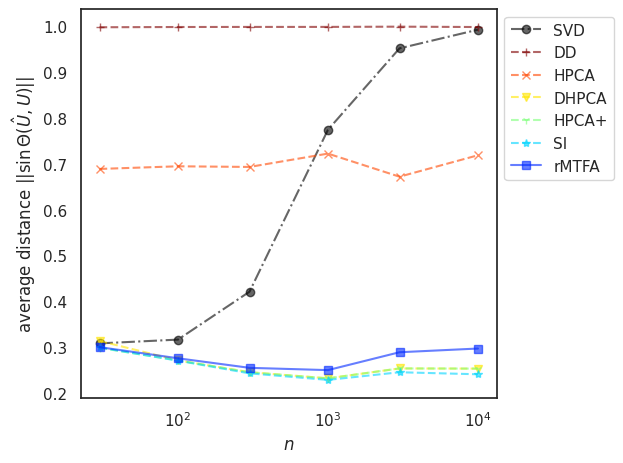} &
\includegraphics[width=0.45\textwidth]{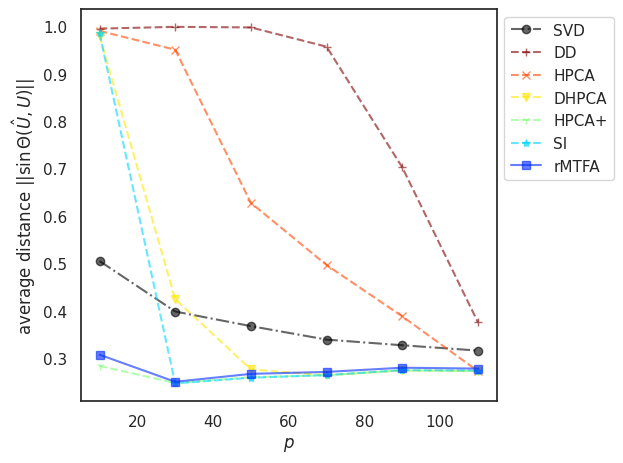} \\
\textbf{(a)} Varying $n$; & \textbf{(b)} Varying $p$;\\[6pt]
\end{tabular}
\begin{tabular}{ccc}
\includegraphics[width=0.45\textwidth]{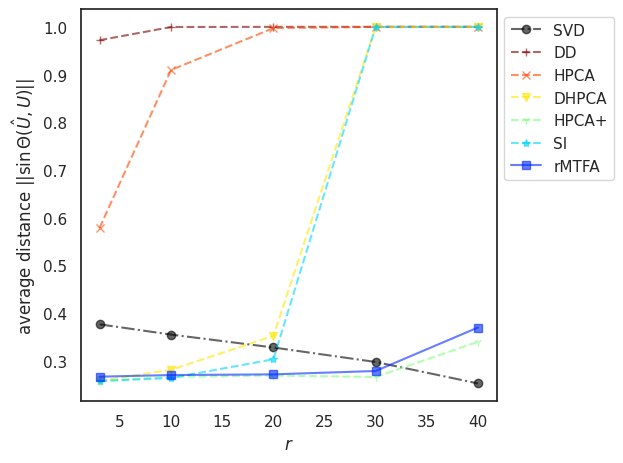} &
\includegraphics[width=0.45\textwidth]{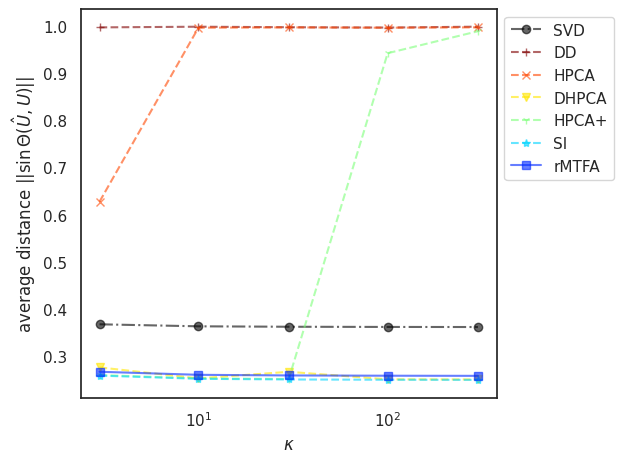} \\
\textbf{(c)} Varying $r$; & \textbf{(d)} Varying $\kappa$; \\[6pt]
\end{tabular}
\begin{tabular}{ccc}
\includegraphics[width=0.45\textwidth]{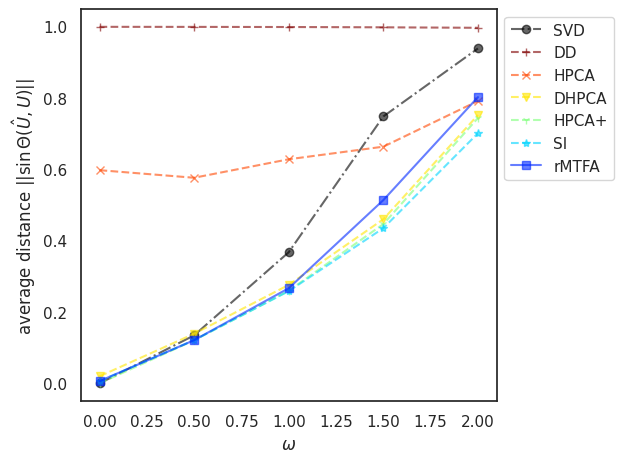} \\
 \textbf{(e)} Varying $\omega$;\\[6pt]
\end{tabular}
\caption{Average $\sin\Theta$ distances for each method, scenario based on 50 simulations. The experiment is conducted with parameters $(n, p, r, \kappa, \omega) = (200, 50, 5, 3, 1)$ while varying one variable at a time. }

\label{fig:numerical}
\end{figure}
The simulation results are depicted in Figure \ref{fig:numerical}. Below, we summarize some key findings. As a reminder, the signal strength increases with $n$ and $p$, specifically $\sigma_r = (np)^{1/4} + p^{1/2}$.
\begin{itemize}
    \item DD consistently exhibits the largest error across all settings, despite the decreasing trend in estimation error with an increase in $p$.
    \item HPCA demonstrates an advantage over SVD as $n$, $p$, and $\omega$ grow. However, it tends to fail rapidly when both $r$ and $\kappa$ increase.
    \item DHPCA successfully mitigates the curse of ill-condition (large $\kappa$) but deteriorates as $r$ increases or $p$ decreases.
    \item HPCA+, with the positive semidefinite constraint, stands out as one of the most robust methods, except for instances with a high spectral gap ($\kappa$).
    \item As a convex relaxation of HPCA, SI performs comparably, if not better, than DHPCA. It handles high condition numbers ($\kappa$) well but is less effective for small $p$ or large $r$.
    \item As a convex relaxation of HPCA+, rMTFA display robustness across all settings, consistently outperforming SVD except when $r$ nears Ledermann bound $\phi(p)$. 
\end{itemize}

As PCA/SVD serves as the foundation for numerous modern methods and applications, including Singular Spectrum Analysis \cite{ghil_advanced_2002} and Principal Component Regression \cite{barshan_supervised_2011, agarwal_robustness_2021}, the introduction of Relaxed MTFA demonstrates both theoretical and empirical advantages over SVD. This advancement not only contributes to the improvement of existing techniques but also opens avenues for more robust and efficient applications in various domains.

\appendix


\bibliographystyle{apalike}

\bibliography{ref.bib}

\section{Appendix: Related Algorithms}

\begin{algorithm}[H]  	
	\begin{algorithmic}
	      	\Require rank $r\in \mathbb{N}$ and  covariance matrix $\Sigma$. 
                        \State $L_\text{dd} = \argmin_{\rank(\mathscr M) \leq r} \norm{\poffdiag({\Sigma}) - \mathscr M}_F^2$ 
	      	\end{algorithmic}
            \caption{Diagonal-deleted PCA \label{alg: diag_deleted_pca}}
    \end{algorithm}

\begin{algorithm}[H]
            
	      	\caption{HeteroPCA$(r, T_{\max}, G^{(0)}= \poffdiag({\Sigma}))$ \cite{zhang_heteroskedastic_2022}\label{alg: Heteropca}}
	      	\begin{algorithmic}

	      		\Require rank $r\in \mathbb{N}$, maximum iteration  $T_{\max}$ and symmetric matrix $G^{(0)}$ taking to be off-diagonal part of covariance matrix, $\poffdiag( {\Sigma})$, by default.
                        
	      		\For{$k = 1, 2, \cdots T_{\max}$} 
                        \State $L ^{(k)} =\argmin_{\rank(\mathscr M) \leq r} \norm{G^{(k-1)}- \mathscr M}_F^2$ \Comment{rank-$r$ partial SVD}
	      		\State $G^{(k)} = \poffdiag({G^{(k-1)}})+\pdiag (L^{(k)}) $
                        
	        		\EndFor 
           
                       \State \Return $(L^{(T_{\max})}$, $G^{(T_{\max})})$  
	      	\end{algorithmic}
	      \end{algorithm}

\begin{algorithm}[H]
     	\caption{Deflated-HeteroPCA \cite{zhou_deflated_2023}\label{alg: deflated_Heteropca}}
	      	\begin{algorithmic}
	      		\Require rank $r\in \mathbb{N}$, maximum iteration $\{T_{\max}^{(k)}\}_{k\in \mathbb{N}}$,  covariance matrix ${\Sigma}$. 
                        \State Set $r_0 = 0, k=0, G^{(0)} = \poffdiag({\Sigma})$
	      		\While{$r_k < r$} 
                        \State $k = k+1$
                        \State $r_k = \sup\left\lbrace r' \in \mathbb{N}\cap (r_{k-1}, r]: \frac{\sigma_{r_{k-1}+1}}{\sigma_{r'}}\leq 4, \frac{\sigma_{r'} - \sigma_{r'+1}}{\sigma_{r'}}\geq r^{-1}\right\rbrace$ \Comment{$\sigma_i$ is the $i$-th singular value of $G^{(k-1)}$ and $\sup \varnothing \defeq r$.}
	      		\State $(L^{(k)}, G^{(k)}) = \text{HeteroPCA}(r_k, T_{\max}^{(k)}, G^{(k-1)})$ \Comment{run Algorithm \ref{alg: Heteropca}}
                        
	        		\EndWhile
	      	\end{algorithmic}
	      \end{algorithm}

\section{Appendix: Proofs}
\subsection{Auxiliary Lemmas}

\begin{lmm}{\label{lmm: constrained_uniqueness}}
    The solution to the constrained optimization problem \eqref{eqn: constrained_form1} is unique.
\end{lmm}
\begin{proof}
    For $\psi \geq \norm{\poffdiag(\Sigma)}_F^2$, the trivial solution $L = 0, D=\pdiag(\Sigma)$ is a unique solution. For $\psi< \norm{\poffdiag(\Sigma)}_F^2$, 
    suppose there are two distinct solutions $x=(L_1, D_1), y = (L_2, D_2)$. By convexity, there is a $t\in(0, 1)$ such that $w = tx+ (1-t)y$ lies in the interior point of the constraint: $\norm{\Sigma - (L+D)} < \psi$ and $w$ is also a minimizer of \eqref{eqn: constrained_form1}. However, it is not possible because of Lemma \ref{lmm:two_optimization} item 2. 
\end{proof}

\begin{lmm} \label{lmm:two_optimization}
    Put $x = (\hat{L}, \hat{D})$. 
    \begin{enumerate}
    \item If $ x$ is the global minimizer of \eqref{eqn:rmtfa} given tuning parameter $\tau>0$ and $\hat L$, then $x$ is also a global minimizer for \eqref{eqn: constrained_form1} given $\psi = \norm{\Sigma - (\hat L + \hat D)}_F^2$.
    \item If $x$ is the global minimizer of \eqref{eqn: constrained_form1} with $0<\psi< \norm{\poffdiag(\Sigma)}_F^2$, then we have $\|{\Sigma} - (\hat{L} +\hat{D})\|_F^2 = \psi$.

    \item If $x$ is the global minimizer of the constrained optimization \eqref{eqn: constrained_form1} with $0 < \psi < \|\poffdiag({\Sigma})\|_F^2$, then there is a $\tau>0$ such  that $x$ is also a global minimizer for \eqref{eqn:rmtfa}.

\end{enumerate}
\end{lmm}
\begin{proof}
The proof is inspired by \cite{li_selecting_2022}. 
    \begin{enumerate}
        \item Indeed, for any $(L, D)$ satisfying the constrain of \eqref{eqn: constrained_form1}:
        \begin{align*}
            \frac{\psi}{2} + \tau \norm{\hat L}_* = \frac{\norm{\Sigma - (\hat L + \hat D)}_F^2}{2} + \tau\norm{\hat L}_*\leq \frac{\norm{\Sigma - (L+D)}_F^2}{2} + \tau \norm{L}_*\leq \frac{\psi}{2} + \tau\norm{L}_*
        \end{align*}
        one has $\norm{\hat L}_* \leq \norm{L}_*$. Moreover, $x = (\hat L, \hat D)$ is in the feasible region of \eqref{eqn: constrained_form1}, therefore $x$ solves \eqref{eqn: constrained_form1}. 
        \item If, by contradiction, $x$ is an interior point of the constraint: $\norm{\Sigma -(\hat L+\hat D)}_F^2 < \psi$. By  Von Nuemann's trace inequality \cite{mirsky_trace_1975}, 
        \begin{align*}
            \sum_{i=1}^p (\sigma_i^{( \Sigma -\hat D)} - \sigma_i^{(\hat L)})^2 \leq \norm{ \Sigma - (\hat L + \hat D)}_F^2 < \psi
        \end{align*}
        The inequality implies that there is $r\in\mathbb{N}, \delta \in (0, \sigma_r^{(\hat L)} )$ such that
        \begin{align*}
            \sum_{i=1}^{r-1} (\sigma_i^{(\Sigma -\hat D)} - \sigma_i^{(\hat L)})^2 +  (\sigma_r^{( \Sigma -\hat D)} - \sigma_r^{(\hat L)}+ \delta)^2  + \sum_{i=r+1}^p (\sigma_i^{( \Sigma -\hat D)} )^2 \leq \psi.    
        \end{align*}
        We can construct $L^\dagger = \hat L - \delta u_r u_r^\top, D^\dagger = \hat D$ where $u_r$ is the unit length eigenvector correspond to $r$th largest eigenvalue of $\hat L$, so that $\norm{L^\dagger}_* < \norm{\hat L}_*$ and $\norm{\Sigma - (L^\dagger + D^\dagger)}_F^2\leq \psi.$ Consequently, we arrive at the conclusion that $x$ cannot be the global minimizer, thereby leading to a contradiction.
        \item 
        
        By applying Slater's constraint qualification \cite{boyd_convex_2004} and noting that the MTFA solution lies in the feasible region of the convex constraint optimization problem \eqref{eqn: constrained_form1}. 

    \end{enumerate}
    \end{proof}

    \begin{lmm}
        $\mathcal{D}_\tau^+(\cdot)$ satisfies for any $M_1, M_2\in\mathbb{S}^p$, 
        \[\norm{ \mathcal{D}_\tau^+ (M_1)- \mathcal{D}_\tau^+(M_2)}_F^2 \leq \norm{M_1-M_2}_F^2\]
        This implies $\mathcal{D}_\tau^+(M)$ is a continuous map in $M$. 
    \end{lmm}
    \begin{proof}
    The argument presented in the proof of Lemma 3 in \cite{mazumder_spectral_2010} remains valid.
    \end{proof}

    \begin{lmm}{\label{lmm:continuous}}
        $(\hat L(\tau), \hat D(\tau))$, the solution to \eqref{eqn:rmtfa} is continuous in $\tau>0$. 
    \end{lmm}
    \begin{proof}
        By Theorem \ref{thm:fixedpt}, $\hat L(\tau)$ solves the fixed point equation $$L = f(L, \tau)\defeq \mathcal{D}_\tau^+ (\poffdiag(\Sigma) + \pdiag(L))$$
        where $f$ is continuous by Lemma \ref{lmm:continuous}.
        At the same time, it solves
        \begin{align*}
            \maximize_{\mathscr L \in  \mathbb{L}}\quad -\norm{\mathscr L - f(\mathscr L, \tau)}_F^2
        \end{align*}
        where $\mathbb{L} = \{L\in\mathbb{S}^{p}_+: \norm{L}_*\leq \norm{\tilde L}_*, \norm{\poffdiag(\Sigma - L)}_F \leq \norm{\poffdiag(\Sigma)}_F\}$ is a compact set. 

        By Berge's Maximum Principle, and the fact that the fixed point equation admit unique solution for each $\tau > 0$, we have the optimizer $\hat L(\tau)$ is continuous in $\tau > 0$.  Hence, $\hat D(\tau) = \pdiag(\Sigma - \hat L(\tau))$ is also continuous in $\tau >0$. 
    \end{proof}

\begin{thm}[Oracle Inequality]\label{thm: oracle}  For any $p\geq 2, n\geq 1$, we have on the event $$\mathcal{F} \defeq\left\lbrace\tau\geq  \norm{\poffdiag(W)}\right\rbrace,$$
it holds that
\begin{multline*}
    \norm{\poffdiag(L-\hat L)}_F^2 \leq
    \\\inf_{\mathscr L \in\mathbb{S}^p_+ }\left\lbrace \norm{\poffdiag(\mathscr L-L)}_F^2 + \min\left( 4 \tau \norm{\mathscr L}_*, 4\tau^2\rank(\mathscr L) + \norm{\pdiag(\hat L - \mathscr L)}_F^2\right) \right\rbrace\\
    \leq \min\left(4\tau \norm{L}_*, 4\tau^2 \rank(L) + \norm{\pdiag(L - \hat L)}_F^2\right) 
\end{multline*}
and 
\begin{equation*}
    \norm{\poffdiag(L-\hat L)} \leq 2 \tau.
\end{equation*}
\end{thm}

\begin{proof}

    For all $\mathscr L  \in \mathbb{S}^p_+$,
    \begin{align*}
        \tau \norm{\hat L}_* + \frac{1}{2}\norm{\poffdiag( \Sigma - \hat L)}_F^2 \leq \tau \norm{\mathscr L}_* + \frac{1}{2}\norm{\poffdiag(\Sigma - \mathscr L)} _F^2
    \end{align*}
    The off-diagonal error can be bounded by
    \begin{align*}
        &\norm{\poffdiag (L - \hat L)} _F^2 \\ \leq& \norm{\poffdiag(\mathscr L - L)}_F^2 + 4\tau\norm{\mathscr L}_*  + 2 \left[\tr\left(\poffdiag(L- \Sigma )\poffdiag(\mathscr L - \hat L)\right)- \tau \norm{\hat L}_*-\tau \norm{\mathscr  L}_*\right]\\
         \leq& \norm{\poffdiag(\mathscr L - L)}_F^2 + 4\tau\norm{\mathscr L}_*  + 2 \left[\tr\left(\poffdiag(L- \Sigma )(\mathscr L - \hat L)\right)- \tau \norm{\mathscr  L- \hat L}_*\right]\\
        \leq & \norm{\poffdiag (\mathscr L - L)}_F^2 + 4\tau \norm{\mathscr L}_*
    \end{align*}
    for all $\mathscr L \in \mathbb{S}^p_+$ where the last inequality holds if $\tau \geq \norm{\poffdiag( \Sigma - L)}$ by Von Neumann's Trace inequality.  

    On the other hand, by sub-gradient optimality condition, there exists $\hat\Gamma \in\partial \norm{\hat L}_*$ such that for all $\mathscr L \in \mathbb{S}^p_+$:
    \begin{align}
        \tr\left((\poffdiag(\hat L -  \Sigma) +\tau \hat \Gamma)(\hat L - \mathscr L)\right)\leq 0. \label{eqn: optimality_foc}
    \end{align}
    Take any $\mathscr L \in\mathbb{S}^p_+$ of rank $r$, with spectral decomposition 
    \begin{align*}
        \mathscr L  = \sum_{i=1}^r \sigma_i u_iu_i^\top 
    \end{align*}
    it follows that for any $\Gamma \in \partial \norm{\mathscr L}_*$:
    \begin{align*}
                & \tr\left(\poffdiag(\hat L - \Sigma) (\hat L - \mathscr L)\right)\\
              \leq &  \tr\left(\poffdiag(\hat L - \Sigma) (\hat L - \mathscr L)\right) +\tau \cdot \tr\left( (\hat \Gamma - \Gamma)(\hat L - \mathscr L)\right) \\ \leq &
               \tr\left((-\tau \Gamma + \poffdiag(  \Sigma -L ))(\hat L - \mathscr L)\right)
    \end{align*}
    where the first inequality is due the monotonicity of sub-differential, and the second inequality is due to \eqref{eqn: optimality_foc}.  By the characterization of sub-differential of the nuclear norm, there exists a $T$ such that $\norm{T} \leq 1$,
    \begin{align*}
        \Gamma  =  \sum_{i=1}^r u_iu_i^\top + \proj{\nul(\mathscr L)} T \proj{\nul(\mathscr L)} = \proj{\col(\mathscr L)} + \proj{\nul(\mathscr L)} T \proj{\nul(\mathscr L)} , 
    \end{align*}
    and \begin{align*}
        \tr\left( \proj{\nul(\mathscr L)} T \proj{\nul(\mathscr L)}(\hat L - \mathscr L)\right) = \norm{\proj{\nul(\mathscr L)} \hat L \proj{\nul(\mathscr L)}}_*.
    \end{align*}
    For such choice of $T$,
    \begin{align*}
                &\frac{1}{2}\left[\norm{\poffdiag(\hat L - L)}_F^2 + \norm{\poffdiag(\hat{L}- \mathscr L)}_F^2  - \norm{\poffdiag(\mathscr L - L)}_F^2\right]\\
                =& \tr\left(\poffdiag(\hat L - \Sigma) (\hat L - \mathscr L)\right)\\
                \leq &
               \tr\left((-\tau \Gamma + \poffdiag( \Sigma - L))(\hat L - \mathscr L)\right) \\
               = &  -\tau\norm{\proj{\nul(\mathscr L)} \hat L \proj{\nul(\mathscr L)}}_*  + \tau \tr\left(\proj{\col(\mathscr L)}(\mathscr L - \hat L)\proj{\col(\mathscr L)}\right) + \tr\left(\poffdiag( \Sigma - L)(\hat L - \mathscr L)\right)\\
               \leq &  
               -\tau\norm{\proj{\nul(\mathscr L)} \hat L \proj{\nul(\mathscr L)}}_*  + \tau \norm{\proj{\col(\mathscr L)}(\hat L - \mathscr L)\proj{\col(\mathscr L)}}_*+ \tr\left(\poffdiag( \Sigma - L)(\hat L - \mathscr L)\right) 
    \end{align*}
    Hence    \begin{multline*}
         \norm{\poffdiag(\hat L - L)}_F^2 + \norm{\poffdiag(\hat{L} - \mathscr L)}_F^2 +2\tau    \norm{\proj{\nul(\mathscr L)} \hat L \proj{\nul(\mathscr L)}}_*\\
        \leq \norm{\poffdiag(\mathscr L - L)}_F^2 + 2\tau \norm{\proj{\col(\mathscr L)} (\hat L - \mathscr L) \proj{\col(\mathscr L)}}_* + 2\tr\left(\poffdiag(\Sigma - L)(\hat L - \mathscr L)\right)
    \end{multline*}
        In particular, \begin{align*}
            \norm{\proj{\col(\mathscr L)} (\hat L - \mathscr L) \proj{\col(\mathscr L)}}_* &\leq \sqrt{\rank(\mathscr L )} \norm{\hat L - \mathscr L}_F = \sqrt{r} \norm{\hat L - \mathscr L}_F
        \end{align*}
        by Von-Numann's trace inequality and Cauchy-Sshwarz's inequality. Similarly,
        \begin{align*}
            &\tr\left(\poffdiag( \Sigma - L)(\hat L - \mathscr L)\right)\\
            \leq &  \tr\left(\poffdiag(\Sigma - L)\proj{\col(\mathscr L)}(\hat L - \mathscr L)\proj{\col(\mathscr L)}\right)+\tr\left(\poffdiag( \Sigma - L)\proj{\nul(\mathscr L)}(\hat L - \mathscr L)\proj{\nul(\mathscr L)}\right)\\
            \leq &  \norm{\poffdiag(\Sigma - L)\proj{\col(\mathscr L)}}_F\norm{\hat L - \mathscr L}_F + \norm{\poffdiag( \Sigma - L)}\norm{\proj{\nul(\mathscr L)}(\hat L - \mathscr L)\proj{\nul(\mathscr L)}}_*\\
            \leq & \left(\sqrt{r}\norm{\hat L - \mathscr L}_F +  \norm{\proj{\nul(\mathscr L)}\hat L \proj{\nul(\mathscr L)}}_*\right)\norm{\poffdiag( \Sigma - L)} 
        \end{align*}
        One has 
        \begin{align*}
         &\norm{\poffdiag(\hat L - L)}_F^2 + \norm{\poffdiag(\hat{L} - \mathscr L)}_F^2 +2\left(\tau - \norm{\poffdiag({\Sigma} -L )}\right)    \norm{\proj{\nul(\mathscr L)} \hat L \proj{\nul(\mathscr L)}}_*\\
        \leq& \norm{\poffdiag(\mathscr L - L)}_F^2 + 2\sqrt{\rank(\mathscr L)}\left(\tau + \norm{\poffdiag(\Sigma - L)} \right) \norm{\hat L - \mathscr L}_F \\
        \leq & \norm{\poffdiag(\mathscr L - L)}_F^2 + \rank(\mathscr L)\left(\tau + \norm{\poffdiag( \Sigma - L)} \right)^2+ \norm{\hat L - \mathscr L}_F^2
    \end{align*}
    So on the event $\tau > \norm{\poffdiag (\Sigma - L)}$,
    \begin{align*}
        \norm{\poffdiag(\hat L - L)}_F^2 \leq \norm{\poffdiag(\mathscr L - L)}_F^2 + 4\tau^2r+ \norm{\pdiag(\hat L - \mathscr L)}_F^2
    \end{align*}    

\end{proof}

\subsection{Section 1}
\subsubsection{Proof of Toy Example in Section \ref{sec:example}}
We can take $\delta^2 =0$. Notice that
\begin{align*}
    v^* = \text{arg max}_{|v|=1}  v^\top \Sigma v\; \Leftrightarrow \; U^\top  v^* = \text{arg max}_{| v|=1}  v^\top ( U^\top \Sigma  U) v 
\end{align*}
for any $p$ by $p$ full rank unitary matrix $ U = ( u_1 \cdots u_p)$. 

Take $ u_1 =  \beta,  u_2 = \frac{\eta - (\beta^\top \eta)  \beta}{\sqrt{1-(\beta^\top\eta)^2}}$. Then
\begin{align*}
     U^\top \Sigma  U/\kappa^2= \begin{pmatrix}(\beta^\top\eta)^2+\sigma^2/\kappa^2 & (\beta^\top\eta)\sqrt{1-(\beta^\top\eta)^2}\\
    (\beta^\top\eta)\sqrt{1-(\beta^\top\eta)^2} & 1-(\beta^\top \eta)^2 \\
    & & 0 \\
    & & & \ddots\\ 
    & & & & 0\end{pmatrix}
\end{align*}
Let $q = \beta^\top \eta, s = \sigma^2/\kappa^2$.
It suffices to solve for the eigenvectors for the 2 by 2 matrix
\begin{align*}
    M = \begin{pmatrix}(\beta^\top\eta)^2+\sigma^2/\kappa^2 & (\beta^\top\eta)\sqrt{1-(\beta^\top\eta)^2}\\
    (\beta^\top\eta)\sqrt{1-(\beta^\top\eta)^2} & 1-(\beta^\top \eta)^2   \end{pmatrix} =  \begin{pmatrix}q^2+s & q\sqrt{1-q^2}\\
    q\sqrt{1-q^2} & 1-q^2   \end{pmatrix} 
\end{align*}
The eigenvalues are
\begin{align*}
    \lambda_\pm = \frac{(1+s)\pm \sqrt{(1-s)^2 + 4sq^2}}{2}
\end{align*}
and eigenvectors are
\begin{align*}
    v_\pm = \begin{pmatrix}q\sqrt{1-q^2}\\\lambda_\pm - (q^2+s)\end{pmatrix}
\end{align*}
so 
\begin{align*}
    v_\pm^* &\propto \beta + \frac{ \lambda_\pm-(q^2+s)}{q({1-q^2})}(\eta - q\beta)\\
    &\propto \beta + \frac{\lambda_\pm -(s+q^2)}{q(1+s-\lambda_\pm)} \eta
\end{align*}
\hfill $\square$
\subsection{Section 2}

\subsubsection{Proof of Lemma \ref{lmm:cvx_uniqueness}} To show convexity of the objective function $F$, consider two points $x = (L_1, D_1), y = (L_2, D_2)$ and parameter $t\in[0, 1]$. Using the Cauchy inequality and AM-GM inequality, we can write:
\begin{align*}
    &\norm{\Sigma - t(L_1+D_1)-(1-t)(L_2+D_2) }_F^2 \\
    =& t^2 \norm{\Sigma - (L_1+D_1)}_F^2 + (1-t)^2 \norm{\Sigma - (L_2+D_2)}_F^2 + 2t(1-t) \tr\left\lbrace(\Sigma - (L_1+D_1))(\Sigma - (L_2+D_2))\right\rbrace \\
    \leq&t^2 \norm{\Sigma - (L_1+D_1)}_F^2 + (1-t)^2 \norm{\Sigma - (L_2+D_2)}_F^2 + 2t(1-t) \norm{\Sigma - (L_1+D_1)}_F\norm{\Sigma - (L_2+D_2)}_F\\
    \leq& t \norm{\Sigma - (L_1+D_1)}_F^2 + (1-t) \norm{\Sigma- (L_2+D_2)}_F^2
\end{align*}
Additionally, we can employ the triangle inequality, which states that $\norm{tL_1 + (1-t)L_2}_* \leq t\norm{L_1}_* + (1-t)\norm{L_2}_*$, to obtain:
\begin{align*}
    F(tx+(1-t)y) =& \tau\norm{tL_1 + (1-t)L_2}_* + \tfrac{1}{2}\norm{\Sigma - t(L_1+D_1)-(1-t)(L_2+D_2) }_F^2\\
    &\leq t F(x) + (1-t) F(y)
\end{align*}
Thus, the objective function is convex. Combing with the fact that the constraint is convex, we proved that the optimization is a convex program.

To establish the uniqueness of the minimum for the convex function $F$, we begin by assuming the contrary, which is that the minimum is achieved at two distinct points, denoted as $x=(L_1, D_1)$ and $y = (L_2, D_2)$.

By convexity, any convex combination of these points, given by $tx + (1-t) y$ for all $t\in[0, 1]$, must also attain the minimum. This necessitates the equalities in the Cauchy inequality and AM-GM inequality to hold, specifically for all $i\neq j$: \[[\Sigma - L_1]_{i,j} = [\Sigma - L_2]_{i,j}. \]
This means that $L_1$ and $L_2$ agree on the off-diagonal cells. Furthermore, we have: $$\norm{\Sigma - (L_1+D_1)}_F^2 = \norm{\poffdiag(\Sigma - L_1)}_F^2 = \norm{\poffdiag(\Sigma - L_2)}_F^2 = \norm{\Sigma - (L_2+D_2)}_F^2. 
$$ This equality arises from the fact that, at the solution $D_i =\pdiag(\Sigma-L_i)$ for $i = 1, 2$. 
Next, let's solve the MTFA problem \eqref{eqn:mtfa} with input  $\poffdiag(L_1)+\pdiag(\Sigma)$ to ensure the solution $w=(\tilde L , \tilde D)$ satisfies: \[\norm{\Sigma - (\tilde L + \tilde D)}_F^2 = \norm{\poffdiag(\Sigma - \tilde L)}_F^2=\norm{\poffdiag(\Sigma - L_1)}_F^2=\norm{\Sigma - (L_1+D_1)}_F^2\]
 and at the same time $\norm{\tilde L}_* \leq \norm{L_1}_* $. Thus,

\[F(w) = \tau \norm{\tilde L}_*+\frac{1}{2}\norm{\Sigma-(\tilde L+\tilde D)}_F^2 \leq  F(x) = F(y).\]
By assumption, $x, y$ are the minimizer of $F$ which implies the equality holds. However, this contradicts the fact that the MTFA solution $w$
must be unique \cite{della_riccia_minimum_1982}. 

\hfill $\square$

\subsubsection{Proof of Lemma \ref{lmm:monotone}}
 $\psi(\tau) = \norm{\Sigma - (\hat L + \hat D)}_F^2$ is a monotone function as it is continuous (Lemma \ref{lmm:continuous}) and one-to-one (Lemma \ref{lmm:two_optimization}) for $0<\tau< \lambda_1^{(\poffdiag(\Sigma))} $.

Finally, consider taking $\Sigma=L, W = D = 0$ in Theorem \ref{thm: oracle}. We obtain the following inequality:
\begin{align*}\label{eqn: off_diag_bound_tau}
    \psi(\tau) = \norm{\poffdiag({\Sigma }-\hat{L}(\tau))}_F^2 \leq 4\tau \norm{ \Sigma}_* 
\end{align*}

\hfill $\square$

\subsubsection{Proof of Theorem \ref{thm:rmta_mtfa}}

For $\tau>0$ small, $\psi(\tau)< \norm{\poffdiag(\Sigma)}_F^2$, by Lemma \ref{lmm:two_optimization}, relaxed MTFA \eqref{eqn:rmtfa} with regularization parameter $\tau$ is equivalent to constrained optimization \eqref{eqn: constrained_form1} with $\psi = \psi(\tau)$. 

Since $\poffdiag(\tilde L) = \poffdiag(\Sigma)$, the matrix $\tilde L$ always lies within the feasible region of \eqref{eqn: constrained_form1} with $\psi = \psi(\tau)$ the above inequality shows the convergence of off-diagonal entries $$\lim_{\tau \rightarrow 0^+} \poffdiag(\hat L(\tau )) = \poffdiag(\tilde L).$$
To demonstrate the convergence of the diagonal entries, we rely on two key facts: uniqueness of the MTFA solution \cite{della_riccia_minimum_1982} and the convergence $\norm{\hat{L}(\tau)}_* \rightarrow \norm{\tilde L}_*$. 
\hfill$\square$

\subsubsection{Proof of Theorem \ref{thm: robust_sin_rmtfa}}\label{sec: robust_sin_rntfa_proof}

    Consider the notation from Algorithm \ref{alg: rmtfa}. 
    Let $U^{(k)} \in \mathbb{O}^{p, r}$ be first $r$ eigenvectors of $G^{(k)} =  \Sigma - D^{(k)}$, $\mathscr U^{(k)} = \col (U^{(k)}), \mathscr U = \col(U)$. Define $b_k = \norm{G^{(k)} - L}$. 
 Since our algorithm converges to a unique fixed point regardless the initial $G^{(0)}$, let's take $G^{(0)} = \poffdiag( \Sigma) + \pdiag(L) = L+\poffdiag(W) $ which means $b_0 = \norm{\poffdiag(W)}$ and our goal is to control the upper bound of $b_k$. By the updating rule of the Algorithm \ref{alg: rmtfa}, we know that $\pdiag(G^{(k+1)}) = \pdiag(L^{(k)}) = \pdiag(\mathcal{D}_\tau^+(G^{(k)}))
 $, and that $\poffdiag(G^{(k+1)}) = \poffdiag(\Sigma)$ for all $k=0, 1, 2, \cdots$. Consequently, 
 \[G^{(k+1)} =  \pdiag(\mathcal D_\tau^+ (G^{(k)})) + \poffdiag(L+W) \]
 and 
 \begin{align*}
      &G^{(k+1)} - L \\
    =& \pdiag(G^{(k+1)}-L)+ \poffdiag(G^{(k+1)}-L)\\
     =& \pdiag(\proj{\mathscr U^{(k)}}(G^{(k)}-L)) - \pdiag(\proj{\mathscr U^{(k)}}^\perp  (L) )  + \pdiag(\mathcal{D}_\tau^+ (G^{(k)})-\proj{\mathscr U^{(k)}}(G^{(k)})) + \poffdiag(W)\\
     =& 
         \pdiag(\proj{\mathscr U}(G^{(k)}-L))+\pdiag((\proj{\mathscr U^{(k)}}-\proj{\mathscr U})(G^{(k)}-L)) 
     \\& \qquad - \pdiag(\proj{\mathscr U^{(k)}}^\perp ( L) )  - \pdiag(\proj{\mathscr U^{(k)}}(G^{(k)})-\mathcal{D}_\tau^+(G^{(k)})) + \poffdiag(W)
 \end{align*}
 By Lemma 1 and 5 in \cite{zhang_heteroskedastic_2022} and the definition of $\mathcal{D}^+_\tau$,\begin{itemize}
     \item $\norm{\pdiag(\proj{\mathscr U}(G^{(k)}-L))}\leq \norm{U}_{2,\infty} \norm{G^{(k)}-L}$
     \item $\norm{\pdiag((\proj{\mathscr U^{(k)}}-\proj{\mathscr U})(G^{(k)}-L))}\leq \norm{(\proj{\mathscr U^{(k)}}-\proj{\mathscr U})(G^{(k)}-L)}\leq \norm{\proj{\mathscr U^{(k)}}-\proj{\mathscr U}}\cdot\norm{G^{(k)}-L}$
     \item $\norm{\proj{\mathscr U^{(k)}}-\proj{\mathscr U}} = \norm{(U^{(k)}_\perp )^\top U }\leq \frac{2\norm{G^{(k)}-L}}{\lambda_r^{(L)}}\wedge 1$
     \item $\norm{\pdiag(\proj{\mathscr U^{(k)}}^\perp (L))}\leq 2\norm{U}_{2, \infty}\norm{G^{(k)}-L}$
     \item $\norm{\pdiag(\proj{\mathscr U^{(k)}}(G^{(k)})-\mathcal{D}_\tau^+(G^{(k)}))}\leq \norm{\proj{\mathscr U^{(k)}}(G^{(k)}) - \mathcal{D}^+_\tau (G^{(k)})}\leq \tau$     
 \end{itemize}
where the last item does not present in the HeteroPCA analysis. 
 
 By triangle inequality, and combine the bounds above. We have for $k\geq 0$,
 \begin{align*}
     b_{k+1}\leq \tau + b_0 + 3\norm{U}_{2, \infty} b_k + \tfrac{2b_k^2}{\lambda_r^{(L)}}
 \end{align*}
 We can bound the $\sin\Theta$ distance in terms of $b_k$
\begin{align*}
    \norm{\sin\Theta(U^{(k)}, U)} &= \norm{(U^{(k)}_\perp)^\top U} \leq \tfrac{\norm{(U_\perp^{(k)})^\top L}}{\lambda_r^{(L)}} \\
    &\leq \tfrac{\norm{(U_\perp^{(k)})^\top G^{(k)}} + \norm{(U_\perp^{(k)})^\top(L-G^{(k)})}}{\lambda_r^{(L)}}\\
    &\leq \tfrac{\sigma_{r+1}^{(G^{(k)})} + \norm{L-G^{(k)}}}{\lambda_r^{(L)}}=  \tfrac{\min_{\rank(\mathscr L)\leq r }\norm{\mathscr L - G^{(k)}} + \norm{L-G^{(k)}}}{\lambda_r^{(L)}}\\
    &\leq \tfrac{2\norm{G^{(k)}-L}}{\lambda_r^{(L)}} = \tfrac{2b_k }{\lambda_r^{(L)}} \leq 2a_k
\end{align*}
where $b_k\leq \lambda_r^{(L)}a_k$ for all $k$, and the sequence $\{a_k\}$ is defined by $a_0 = \frac{\tau+b_0}{\lambda_r^{(L)}}$,\begin{align*}
    a_{k+1} = a_0 + 3\norm{U}_{2, \infty} a_k + 2 a_k^2. 
\end{align*}
 for all $k\in\mathbb{N}$. 
 Consider $\epsilon_1, \epsilon_2, \varrho>0$ such that
 \[3\norm{U}_{2, \infty} + 2(1-\varrho)^{-1} a_0 < \epsilon_1<\epsilon_2 <\varrho < 1\]
 We show $a_k \leq \frac{1}{1-\varrho}a_0 + (\epsilon_2-\epsilon_1)\varrho^k/2$ by induction. Indeed, it holds for $k = 0$. Assume it holds for $k=m$.  Then for $k = m+1$, 
 \begin{align*}
     a_{m+1} &\leq  a_0 + [3\norm{U}_{2, \infty} + 2((1-\varrho)^{-1}a_0+(\epsilon_2-\epsilon_1)\varrho^m/2)]a_m\\
     &\leq a_0 + [\epsilon_1+ (\epsilon_2-\epsilon_1)\varrho^m]a_m\\
     &\leq a_0 + \varrho a_m\\
     &\leq a_0 + \varrho [(1-\varrho)^{-1}a_0+(\epsilon_2-\epsilon_1)\varrho^m/2]\\
     &\leq \tfrac{1}{1-\varrho }a_0+(\epsilon_2-\epsilon_1)\varrho^{m+1}/2
 \end{align*}
As a result, $$\norm{\sin\Theta(\hat U, U)}\leq \limsup_{k\rightarrow\infty} 2a_{k} \leq  \tfrac{2a_0}{1-\varrho}=\tfrac{2}{1-\varrho}\cdot \tfrac{\tau + \norm{\poffdiag(W)}}{\lambda_r^{(L)}}.$$

\hfill $\square$

\subsection{Section 3}

\subsubsection{Proof of Theorem \ref{thm:factor_subspace}}
\label{sec:factor_subspace_pf}


The proof in \cite{zhang_heteroskedastic_2022} still go through, but in Step 3, invoke our robust $\sin \Theta$ Theorem (Theorem \ref{thm: robust_sin_rmtfa}) instead. 

\subsubsection{Proof of Lemma \ref{lmm:heywood}}

For $\tau$ large enough ($\tau > \bar \tau \defeq \norm{\poffdiag(\Sigma)}$), the trivial solution $(\hat L(\tau), \hat D(\tau)) = (0, \pdiag(\Sigma))$ minimizes \eqref{eqn:rmtfa}. Hence $\hat D_{i, i} = \Sigma_{i, i} >0$. By continuity (Lemma \ref{lmm:continuous}), there is an interval of values of $\tau< \bar \tau$ such that the corresponding solution is non-trivial $ \min _{1\leq i\leq p} [\hat D(\tau)]_{i, i} >0 $. 

\subsubsection{Proof of Theorem \ref{thm: svd_hetero}}
\label{sec: svd_hetero_pf}

Define $\mu = \max(n\norm{V}_{2, \infty}, p\norm{U}_{2, \infty})/r, d = p\vee n$. Consider $\dot U \in \mathbb{O}^{p\times r}, \dot W \in \mathbb{O}^{r\times r}, \dot\Lambda= \diag(\dot \sigma_1, \dots \dot \sigma_r), \dot\sigma_1\geq \cdots \geq \dot\sigma _r > 0$ such that $$\dot U \dot \Lambda \dot W^\top = U\Lambda+ZV$$  

Weyl's theorem yields:
\begin{align*}
    \max_{1\leq i\leq r}|\sigma_i - \dot \sigma _ i| \leq \norm{ZV} \lesssim \sqrt{(\mu r \omega_{\max{}}^2 + \omega_{\col}^2)\log (d)}
\end{align*}
where the latter inequality hold with probability at least $1-O(d^{-10})$. It is shown by \cite{zhou_deflated_2023}. Their Assumption 1 is a general setting compared to Assumption \ref{assmpt: error_term}. Now,
\begin{align*}
    \norm{\dot\Sigma^{-1}} = \frac{1}{\dot\sigma_r} \leq \frac{1}{\sigma_r - \max_{1\leq i \leq r} |\sigma_i-\dot \sigma_i|}
    \leq \frac{2}{\sigma_r} 
\end{align*}
provided $\sigma_r \geq C_0 \sqrt{(\mu r \omega_{\max}^2 + \omega_{\col}^2)\log (d)}$ for some $C_0$ sufficiently large. Likewise,
\begin{align*}
    \sigma_r \leq \dot\sigma_r + \max_{1\leq i\leq r} |\sigma_i - \dot\sigma_i| \leq \dot\sigma_r + \frac{\sigma_r}{2}
\end{align*}
Therefore, $$\frac{\dot\sigma_r}{2}\leq\sigma_r\leq 2\dot\sigma_r. $$

By the identity of $\sin\Theta$ distance and triangular inequality:
\begin{align*}
    \norm{\sin \Theta(\hat U, U)} &= \norm{\hat U\hat U^\top  - UU^\top} \\
    &\leq \norm{\hat U \hat U^\top - \dot U\dot U^\top} + \norm{U U^\top - \dot U\dot U^\top} \\
    &= \norm{\sin\Theta(\hat U, \dot U)} + \norm{\sin\Theta(\dot U, U)}
\end{align*}

We first bound $\norm{\sin\Theta(\dot U, U)}$ via Wedin's $\sin \Theta$ theorem (Theorem 2.9 in \cite{chen_spectral_2021}), 
\begin{align*}
    \norm{\sin \Theta (\dot U, U)} \lesssim \frac{\norm{ZV}}{\sigma_r} \lesssim \frac{\sqrt{(\mu r\omega_{\max}^2 + \omega_{\col}^2 )\log(d)}}{\sigma_r}
\end{align*}
with probability at least $1-O(d^{-10})$. 

Next, to bound $\norm{\sin\Theta (\hat U, \dot U)}$, we notice that \begin{align*}
    YY^\top &= (M+Z)(M+Z)^\top\\
    &= (U\Lambda +ZV)(U\Lambda+ZV)^\top + (ZZ^\top - ZVV^\top Z^\top)\\
    &= \dot U \dot \Lambda^2 \dot U^\top + (ZZ^\top -ZVV^\top Z). 
\end{align*}
We can invoke Theorem \ref{thm: robust_sin_rmtfa} provided event $\mathcal E_\varrho$ stated in the theorem happens, i.e.~$\norm{\dot U}_{2, \infty}$ is bounded by some small constant and $\norm{\poffdiag(ZZ^\top - ZVV^\top Z)}\ll \dot\sigma_r^2$. 

From $(\dot U - UU^\top \dot U) \dot \Lambda \dot W^\top =(I-UU^\top)ZV$, one has
\begin{align*}
    \norm{\dot U - UU^\top \dot U}_{2, \infty} &= \norm{(I-UU^\top)ZV\dot W \dot \Lambda^{-1}}_{2, \infty}\\
    &\leq \left( \norm{ZV}_{2, \infty} + \norm{UU^\top ZV}_{2, \infty}\right)\norm{\dot \Lambda^{-1}}\\
    &\lesssim \frac{(\omega_{\row}+\omega_{\col})\sqrt{\log (d)}}{\sigma_r}\ll 1
\end{align*}
with probability at least $1-O(d^{-10})$ provided $\sigma_r\gg \sqrt{(\mu r\omega_{\max}^2+ \omega_{\col}^2)\log(d)}$. Thus,
\begin{align*}
    \norm{\dot U}_{2, \infty} &\leq \norm{\dot U  - UU^\top \dot U}_{2, \infty} + \norm{UU^\top \dot U}_{2, \infty}\\&\leq \norm{\dot U - UU^\top \dot U}_{2, \infty} + \norm{U}_{2, \infty} \norm{U^\top \dot U}\\
    &\leq c.
\end{align*}
for some constant $c>0$.
Also,
\begin{align*}
    \norm{\poffdiag(ZZ^\top - ZVV^\top Z)} &\leq \norm{\poffdiag(ZZ^\top ) }+ 2\norm{ZV}^2\\
    &
    \lesssim  \omega_{\col}(\omega_{\row}+ \omega_{\col})\log (d)+ (\mu r\omega_{\max}^2+ \omega_{\col}^2)\log(d)
    \ll \sigma_r^2
    \end{align*}
    with probability at least $1-O(d^{-10})$.
    By Theorem \ref{thm: robust_sin_rmtfa},
    $$\norm{\sin \Theta(\hat U, \dot U)}\lesssim \left(\tfrac{\sqrt{(\mu r \omega_{\max}^2 + \omega_{\col}^2)\log(d)}}{\sigma_r}\right)^2 + \tfrac{\omega_{\col}\omega_{\row}\log (d)}{\sigma_r^2}$$
    with high probability. 

Combine the results together, taking expectation, one has
\begin{align*}
   \e \norm{\sin\Theta(\hat U, U)} &\lesssim  \left(\tfrac{\sqrt{(\mu r \omega_{\max}^2 + \omega_{\col}^2)\log(d)}}{\sigma_r}+ \tfrac{\omega_{\col}\omega_{\row}\log (d)}{\sigma_r^2}\right)(1-O(d^{-10})) + O(d^{-10})\\
   &\approx \tfrac{\sqrt{(\mu r \omega_{\max}^2 + \omega_{\col}^2)\log(d)}}{\sigma_r}+ \tfrac{\omega_{\col}\omega_{\row}\log (d)}{\sigma_r^2}\\
   &\lesssim \tfrac{\omega_{\col}\sqrt{\log(d)}}{\sigma_r} + \tfrac{\omega_{\col}\omega_{\row} \log(d)}{\sigma_r^2}
\end{align*}
The last inequality is by assumption: $\omega_{\col}^2/\omega_{\max}^2\gtrsim \mu r$. \hfill $\square$

\subsection{Section 4}

\subsubsection{Proof of Lemma \ref{lmm: spectral_thresholding}}\label{subsection: proof_spectral_thresholding}

By spectral decomposition, a symmetric matrix can be represented as difference of positive semi-definite matrices: $ X =  X^+- X^-$ where $X^+, X^-\in \mathbb{S}^p_+$ and $\tr(X^+ X^-) = 0$. Consider $M \in \mathbb{S}^p_+$, by $\tr(X^-M)\geq 0$ and Von Nuemann's trace inequality,

\begin{align*}
    &\quad\tau\norm{M}_*+\tfrac{1}{2}\norm{ X -  M}_F^2\\ &= \tau \norm{ M}_* + \tfrac{1}{2}\norm{ X^+ -  M}_F^2 + \tfrac{1}{2}\norm{ X^-}_F^2-\tr( X^-( X^+-M))\\
    &=  \tau \norm{ M }_* + \tfrac{1}{2}\norm{ X^+ - M }_F^2 + \tfrac{1}{2}\norm{X^-}_F^2+\tr( X^- M )\\
    &\geq \tau \norm{ M }_* + \tfrac{1}{2}\norm{ X^+ - M }_F^2+ \tfrac{1}{2}\norm{ X^-}_F^2\\
     &\geq \sum_i \left\lbrace  \tfrac{1}{2}(\sigma_i^{(X^+)}-\sigma_i^{(M)})^2+\tau |\sigma_i^{(M)}|\right\rbrace+  \tfrac{1}{2}\norm{X^-}_F^2\\
     &\geq \sum_i \min_{\sigma_i^{(M)}\geq 0}\left\lbrace  \tfrac{1}{2}(\sigma_i^{(X^+)}-\sigma_i^{(M)})^2+\tau |\sigma_i^{(M)}|\right\rbrace+  \tfrac{1}{2}\norm{X^-}_F^2
\end{align*}

The last term attains its minimum when $\sigma_i^{(M)} = (\sigma_i^{(X^+)}-\tau)_+$, and the equal signs preceding hold true if and only if $X^+$ and $M$ share the same eigenspace. Consequently, $ M =\mathcal{D}_\tau^+(X)$ stands as the unique minimizer. \hfill$\square$

\subsubsection{Proof of Theorem \ref{thm: convergence_beck}}\label{subsection: relaxed_mtfa_convergence}
\begin{proof}
    Use Theorem 3.9 in \cite{beck_convergence_2015} with \begin{align*}
        x &= (y, z) = (L, D),\\ H(y, z) &= f(y, z)+g_1(y)+g_2(z),\\
        f(y, z) &= \norm{ \Sigma -y-z}_F^2\\
        g_1(y) &=\tau\norm{y}_* + \infty \cdot 1_{\{y\not \in \mathbb{S}^p_+\}}\\
        g_2(z) &= \infty \cdot 1_{\{ z \neq \pdiag (z)\}}
    \end{align*}
    Conditions [A]-[E] are satisfied:
    \begin{enumerate}
        \item[][A] The functions $g_i:\mathbb{R}^{p^2}\rightarrow (-\infty, \infty], i=1, 2$ are closed and
proper convex functions that are subdifferentiable over their domain.  
\item[][B] The function $f$ is a continuously differentiable convex function over $\text{dom} (g_1) \times 
\text{dom} (g_2)$
\item[][C], [D] The gradient of $f$ is (uniformly) Lipschitz continuous with respect to the
variables vector $y, z$ over $\text{dom}(g_1), \text{dom}(g_2)$, respectively with constant $L_1 = L_2 = 2; $ since
\begin{align}
    \nabla_y^2 f =\nabla_z^2f = 2I\preceq 2I
\end{align}
\item[][E] The optimal set of solution is nonempty and for any $\bar y\in\text{dom}(g_1), \bar z \in \text{dom}(g_2)$, the problems
\begin{align*}
    \min_z f(\bar y, z)+g_2(z),\quad \min_y f(y, \bar z)+ g_1(y) 
\end{align*}
have minimizers.
\end{enumerate}
\end{proof}

\end{document}